\documentclass{article}
\usepackage[preprint]{neurips_2023}
                        
\usepackage[american]{babel}

\usepackage{natbib} 
\usepackage{mathtools} 
\usepackage{booktabs} 
\usepackage{tikz} 

\author{Vihari Piratla\\
University of Cambridge\\
\texttt{vp421@cam.ac.uk}
\And
Juyeon Heo\thanks{equal contribution}, Katherine Collins\footnote[1]{}, Sukriti Singh\footnote[1]{}\\
University of Cambridge 
\And
Adrian Weller \\
University of Cambridge \\
Alan Turing Institute 
}
\usepackage{algorithm}
\usepackage{algpseudocode}
\usepackage{amsmath, amsfonts, amssymb}
\usepackage{subcaption}
\usepackage{listings}[numbers=none]
\usepackage{wrapfig}
\usepackage{tikz}
\usepackage{amsthm}
\usepackage{inconsolata}

\newtheorem{prop}{Proposition}
\newtheorem{cor}{Corollary}

\usepackage{hyperref}
\usepackage{soul}
\usepackage{float}
\usepackage{listings}
\usepackage{xcolor}
\usepackage{makecell}
\usepackage{enumitem}
\usepackage{caption}
\DeclareCaptionStyle{ruled}{labelfont=normalfont,labelsep=colon,strut=off} 
\title{Estimation of Concept Explanations Should be Uncertainty-Aware}
\newcommand{\mex}{model-to-be-explained}
\newcommand{\vx}{\mathbf{x}}

\newcommand{\vu}{\mathbf{u}}
\newcommand{\vv}{\mathbf{v}}
\newcommand{\vw}{\mathbf{w}}
\newcommand{\csim}{\text{cos-sim}}

\newcommand{\ours}{U-ACE}

\newcommand{\oursfull}{Uncertainty-Aware Concept Explanations}
\newcommand{\cs}{concept activations}
\newcommand{\simple}{\textit{Oracle}}

\newcommand{\new}[1]{\textcolor{black}{#1}}
\newcommand{\susi}[1]{} 
\newcommand{\vp}[1]{} 
\definecolor{bleudefrance}{rgb}{0.19, 0.55, 0.91}
\newcommand{\jh}[1]{}
\newcommand{\aw}[1]{}

\DeclareMathOperator*{\argmax}{arg\,max}
\DeclareMathOperator*{\argmin}{arg\,min}
\begin{document}
\maketitle

\begin{abstract}
    Model explanations can be valuable for interpreting and debugging predictive models.
    We study a specific kind called Concept Explanations, where the goal is to interpret a model using human-understandable concepts.
    Although popular for their easy interpretation, concept explanations are known to be noisy. 
    We begin our work by identifying various sources of uncertainty in the estimation pipeline that lead to such noise. 
    We then propose an uncertainty-aware Bayesian estimation method to address these issues, which readily improved the quality of explanations. 
    We demonstrate with theoretical analysis and empirical evaluation that explanations computed by our method are robust to train-time choices while also being label-efficient.     
    Further, our method proved capable of recovering relevant concepts amongst a bank of thousands, in an evaluation with real-datasets and off-the-shelf models, demonstrating its scalability.
    We believe the improved quality of uncertainty-aware concept explanations make them a strong candidate for more reliable model interpretation. We release our code at \url{https://github.com/vps-anonconfs/uace}.

\end{abstract}

\section{Introduction}
\label{sec:intro}

As increasingly complex machine learning (ML) systems proliferate into real-world decision-making and reasoning, there is a growing need to be able to explain such systems. Concept-based explanations are a class of interpretable methods that explain predictions using high-level and semantically meaningful concepts~\citep{TCAV}. Concept explanations are aligned with how humans communicate their decisions~\citep{Yeh2022} and are shown~\citep{TCAV,Kim2023CHI} to be more preferable over explanations using salient input features~\citep{Ribeiro2016,Selvaraju2017} or salient training examples~\citep{Koh2017}. Concept explanations also show promise for encoding task-specific prior knowledge~\citep{PosthocCBMICLR2023} and propelling scientific discovery~\citep{Yeh2022}.

Concept explanations offer insight into a pretrained prediction model by estimating the importance of concepts using two human-provided resources: (1) a list of potentially relevant concepts for the task, and (2) a dataset of examples, often referred to as the ``probe-dataset''. Estimation typically proceeds in two steps: (a) compute the log-likelihood of a concept given an example called \cs{}, and (b) aggregate their per-example activation scores into a globally relevant explanation. For example, the concept {\it wing} is considered important if the information about the concept is encoded in all examples of the {\it plane} class in the dataset. Owing to their model-level granularity, concept explanations are easy to interpret and have witnessed wide recognition in diverse applications~\citep{Yeh2022}.

Notwithstanding their popularity, concept explanations are known to be noisy and data expensive.
\citet{Ramaswamy2022} showed that existing estimation methods are sensitive to the choice of concept set and dataset raising concerns over their interpretability. Another major limitation of concept-based explanation is the need for datasets with explicit concept annotations. Increasingly popular multimodal models such as CLIP~\citep{clip} present an exciting alternate direction to specify relevant concepts, especially for common image applications, through their text description. Recent work has explored using multimodal models for training Concept Bottleneck Models (CBMs)~\citep{Oikarinen2023,PosthocCBMICLR2023,Moayeri23}, but they are not yet evaluated for generating post-hoc concept explanations. 

Our objective is to improve the quality of concept explanations while also not requiring datasets with concept annotations. We begin by observing that existing estimation methods do not model uncertainty in the estimation pipeline, leading to high variance or noisy explanations. We identify at least \textit{two sources of uncertainty} in the standard estimation of concept explanations. First, \textit{when a concept is missing from the probe-dataset}, we cannot estimate its importance with confidence. Reporting uncertainty over estimated importance of a concept can thus help the user draw a more informed interpretation. Second, \textit{when a concept is hard to recognize, or irrelevant to the task}, the corresponding activations predicted from the representation layer of the \mex{} are expected to be noisy. 
If not modelled, uncertainty over concept activations either due to their absence, hardness, or relevance cascades into noise in explanations. Appreciating the need to model uncertainty, we present an estimator called \textbf{\oursfull{} (\ours{})}, which we show is instrumental in improving quality of explanations.  

\paragraph{Contributions.}
$\bullet$ We motivate the need for modelling uncertainty for faithful estimation of concept explanations. 
$\bullet$ We propose a \textbf{Bayesian estimation method} called~\ours{} that is both \textbf{label-free} and \textbf{models uncertainty in the estimation of concept explanations}.  
$\bullet$ We demonstrate the merits of our proposed method~\ours{} through theoretical analysis and empirical evidence on two controlled and three real-world datasets. 

\section{Background and Motivation}
\label{sec:background}
We denote the model-to-be explained as $f:\mathbb{R}^D\rightarrow \mathbb{R}^L$ that maps D-dimensional inputs to L labels. Let $f^{[l]}(\vx)$ denote the $l^{th}$ layer representation space and $f(\vx)[y]$ for $y\in [1, L]$ be the logit for the label $y$. Given a probe-dataset of examples $\mathcal{D}=\{\vx^{(i)}\}_{i=1}^N$ and a list of concepts $\mathcal{C}=\{c_1, c_2, \dots, c_K\}$, our objective is to explain the pretrained model $f$ using the specified concepts. Traditionally, the concepts are demonstrated using potentially small and independent datasets with concept annotations $\{\mathcal{D}_c^{k}: k\in [1, K]\}$ where $\mathcal{D}_c^{k}$ is a dataset with positive and negative examples of the $k^{th}$ concept.

Concept-Based Explanations (CBE) proceed in two steps. In the first step, they learn \textbf{concept activation vectors (CAVs)} that predict the concept from $l^{th}$ layer representation of an example. More formally, they learn the concept activation vector $v_k$ for $k^{th}$ concept by optimizing $v_k = \argmin_{v}\mathbb{E}_{(x, y)\sim \mathcal{D}_c^{(k)}}[\ell(v^Tf^{[l]}(\vx), y)]$ where $\ell$ is the usual cross-entropy loss. The inner product of the representation with the concept activation vector $v_k^Tf^{[l]}(\vx)$ is usually referred to as the \textbf{\cs{}}. Various approaches exist to aggregate example-specific concept activations into model-level explanations for the second step. \citet{TCAV} computes sensitivity of logits to interventions on concept activations to compute what is known as the CAV score per example per concept and report the fraction of examples in the probe-dataset with a positive CAV score as the global importance of the concept known as TCAV score.
\citet{Zhou2018} proposed to decompose the classification layer weights as $\sum_k \alpha_k v_k$ and report the coefficients $\alpha_k$ as the importance score of the $k^{th}$ concept. We refer the reader to~\citet{Yeh2022} for an in-depth survey. 

{\bf Data-efficient concept explanations.} A major limitation of traditional CBEs is their need for datasets with concept annotations $\{\mathcal{D}_c^{1}, \mathcal{D}_c^{2}, \dots\}$. In practical applications, we may wish to find important concepts among thousands of potentially relevant concepts, which is not possible without expensive data collection. Recent proposals~\citep{PosthocCBMICLR2023,Oikarinen2023,Moayeri23} suggested using pretrained multimodal models like CLIP to evade the data annotation cost for a related problem called Concept Bottleneck Models (CBM)~\citep{Koh2020}. CBMs aim to train inherently interpretable model with a concept bottleneck. Although CBMs cannot generate explanations for a \mex{}, a subset of methods propose to train what are known as Posthoc-CBMs using the representation layer of a pretrained task model for data efficiency. Given that Posthoc-CBMs base on the representation of a pretrained task model, we may use them to generate concept explanations. We describe briefly two such CBM proposals below. 

\citet{Oikarinen2023} (O-CBM) estimates the concept activation vectors by learning to linearly project from the representation space of CLIP where the concept is encoded using its text description to the representation space of the \mex{} $f$. It then learns a linear classification model on \cs{} and returns the weight matrix as the concept explanation.
Based on the proposal of \citet{PosthocCBMICLR2023}, we can also generate explanations by training a linear model to match the predictions of \mex{} directly using the concept activations of CLIP, which we denote by (Y-CBM). 

{\bf Noisy explanations.} Alongside data inefficiency, concept explanation methods are known to be noisy. We observed critical  concerns with existing CBEs in the same spirit as the challenges raised in \citet{Ramaswamy2022}. As we demonstrate in Sections~\ref{sec:expt:stl},~\ref{sec:expt:real},~\ref{sec:simstudy}, concept explanations for the same \mex{} vary with the choice of the probe-dataset or the concept set bringing into question the reliability of explanations.  

\begin{figure*}
    \includegraphics[width=0.96\linewidth]{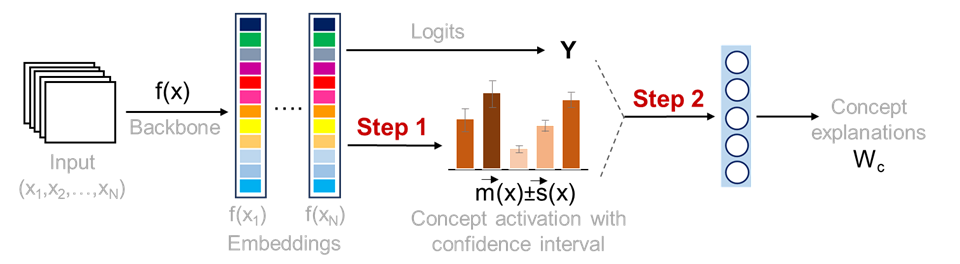}
    \caption{Our proposed estimator: \oursfull{} (\ours{}). We track uncertainty in concept activation scores from Step 1, and model them in Step 2.}
    \label{fig:overview}
\end{figure*}

\section{\oursfull{}}
\label{sec:method}
As summarized in the previous section, CBEs rely on \cs{} for generating explanations. It is not hard to see that the activation score of a concept cannot be predicted confidently if the concept is hard/ambiguous, or if it is not encoded by the \mex{}. 
Moreover, the importance of a concept cannot be confidently estimated if it is missing from the probe-dataset. 
Motivated by the need to model uncertainty in the estimation, we design our estimator: \ours{}.

Our approach encompasses the following steps: (1) estimate \cs{} along with their error interval, and (2) aggregate concept activations and their confidence intervals in to a global concept explanation. We describe the estimation of \cs{} and their error given an instance $\vx$ denoted as $\vec{m}(\vx), \vec{s}(\vx)$ respectively in Section~\ref{sec:method:noise}. By definition, the true concept activation for a concept $k$ and instance $\vx$ is in the range of $\vec{m}(\vx)\pm \vec{s}(\vx)$ with a high probability. \new{We describe the estimation of concept explanations in what follows using $\vec{m}(\vx), \vec{s}(\vx)$, which is independent of how they are computed.} \jh{Can you make it clear? the estimation of concept explanations is independent of how $\vec{m}(\vx), \vec{s}(\vx)$ are computed? What do you mean?}

\new{We compute explanations by fitting a linear regression model on the concept activations in the same spirit as many CBM methods as it is easier to incorporate the input noise in a regression model.} 
Our objective is to learn linear model weights $W_c$ of size $L\times K$ (recall that L, K are the number of labels and concepts respectively) that map the \cs{} to their logit scores, i.e. $f(\vx)\approx W_c\vec{m}(\vx)$. Since the \cs{} contain noise, we require that $W_c$ is such that predictions do not change under noise, that is $W_c[\vec{m}(\vx) + \vec{s}(\vx)]\approx W_c\vec{m}(\vx)\implies W_c\vec{s}(\vx)\approx 0$. I.e. the inner product of each row ($\vec{w}$) of $W_c$ with $\vec{s}(\vx)$ must be negligible. For the sake of exposition, we analyse the solution of y$^{th}\in [1, L]$ row $\vec{w}$ of $W_c$, which can be easily generalized to the other rows. \new{We cast the bounded error constraint, i.e. $|\vec{w}^T\vec{s}(\vx)|\leq \delta$ for some small positive $\delta$ and for all the instances $\vx$ in the probe-dataset, into a distributional prior over the weights as shown below.} 

The per-example constraint $|\vec{w}^T\vec{s}(\vx)|\leq \delta$ leads to the following constraint with the average noise. 
$$
  |\vec{w}^T\epsilon| \leq \frac{\sum_{\vx\in \mathcal{D}} |\vec{w}^T\vec{s}(\vx)|}{N} \leq \delta \text{ where }\epsilon \triangleq \frac{\sum_{x\in \mathcal{D}}\vec{s}(\vx)}{N}
$$
If we require that the dot-product with the average noise be bounded, i.e. $|\vec{w}^T\epsilon| \leq \delta, \text{ for some small }\delta>0$ with high probability, their norm, $\vec{w}^T\epsilon\epsilon^T\vec{w}$, must be bounded. If we then approximate $\epsilon\epsilon^T$ with the positive semi-definite $\text{diag}(\epsilon\epsilon^T)$, we can pose the constraint as a distributional prior as below. 
\begin{align*}
  & \vec{w}^T\text{diag}(\epsilon\epsilon^T)\vec{w} \leq \delta^2\\
  &\Rightarrow -\frac{1}{2}(\vec{w}-\mathbf{0})^TS^{-1}(\vec{w}-\mathbf{0}) \text{ where }S^{-1}=\text{diag}(\epsilon\epsilon^T) \\
  &\text{ is high when $\vec{w}$ satisfies the constraint}\\
  & \Rightarrow \mathcal{N}(\vec{w}; \mathbf{0}, \lambda S) \text{ is high for an appropriate $\lambda>0$}\\
  &\Rightarrow \vec{w} \sim \mathcal{N}(\mathbf{0}, \lambda S)
\end{align*}
We observe therefore that the weight vectors drawn from $\mathcal{N}(\mathbf{0}, \lambda \text{diag}(\epsilon\epsilon^T)^{-1})$ satisfy the invariance to input noise constraint with high probability.
We now estimate the posterior on the weights after having observed the data with the prior on weights set to $\mathcal{N}(0, \lambda \text{diag}(\epsilon\epsilon^T)^{-1})$. The posterior over weights has the following closed form\citep{SMLRuss} where $C_X=[\vec{m}(\vx_1), \vec{m}(\vx_2), \dots, \vec{m}(\vx_N)]$ is a $K\times N$ matrix and $Y=[f(\vx_1)[y], f(\vx_2)[y], \dots, f(\vx_N)[y]]^T$ is an $N\times 1$ vector (\new{derivation in Appendix~\ref{sec:appendix:posterior}}). 
\begin{align}
&\Pr(\vec{w}\mid C_X, Y) = \mathcal{N}(\vec{w}; \mu, \Sigma)\label{eq:posterior}\\
&\text{ where } \mu = \beta\Sigma C_XY, \quad \Sigma^{-1} = \beta C_XC_X^T + \lambda^{-1} \text{diag}(\epsilon\epsilon^T)\nonumber
\end{align}
$\beta$ is the inverse variance of noise in observations Y. 
We optimise both $\beta$ and $\lambda$ using MLE on $\mathcal{D}$ (more details in Appendix~\ref{sec:appendix:mle}). We could directly set the inverse of $\beta$ approximately 0 since there is no noise on the observations Y. Instead of setting $\beta$ to an arbitrary large value, we observed better explanations when we allowed the tuning algorithm to find a value of $\beta, \lambda$ to balance the evidence and noise.



The estimator shown in Equation~\ref{eq:posterior} is how we model noise in \ours{}. In the next section, we describe how we may estimate the noise in the \cs{}. Algorithm~\ref{alg:ours} summarizes our proposal.

\subsection{Estimation of  \cs{} and their noise}
\label{sec:method:noise}
\new{In this section, we discuss how we estimate $\vec{m}(\vx), \vec{s}(\vx)$ using a pretrained multimodal model.}
Recall that image-text multimodal (MM) systems such as CLIP~\citep{clip} can embed both images and text in a shared representation space, which enables one to estimate the similarity of an image to any phrase. This presents us with the flexibility to specify a concept using its text description ($T_k$ for the $k^{th}$ concept) without needing concept datasets $\mathcal{D}_c^k$. We denote by $g(\bullet)$ the image embedding function of MM and $g_{text}(\bullet)$ the text embedding function. 

\new{Our objective is to estimate $\vec{m}(\vx), \vec{s}(\vx)$ such that the true concept activation value is in the range $\vec{m}(\vx)\pm \vec{s}(\vx)$. Two major sources of uncertainty in concept activations that must inform $s(\vx)$ are due to (1) {\it epistemic (model) uncertainty} arising from lack of information about the concept in the representation layer of the \mex{}, (2) {\it data uncertainty} arising from ambiguity (because the concept is not clearly visible, see Figures~\ref{fig:amb:1}, \ref{fig:amb:2}, \ref{fig:amb:3}, \ref{fig:amb:4} of Appendix~\ref{sec:appendix:uncert} for some examples). We wish to estimate $\vec{s}(\vx)$ that is aware of both the forms of uncertainty.} 

\new{We can obtain a point estimate for the activation vector of the $k^{th}$ concept $v_k$ such that $f(\vx)^Tv_k\approx g(\vx)^Tw_k$ (where $w_k=g_{text}(T_k)$) for all $\vx$ in the probe-dataset $\mathcal{D}$ through simple optimization~\citep{Oikarinen2023,Moayeri23}. We may then simply repeat the estimation procedure multiple times to sample from the distribution of activation vectors and their corresponding \cs{}. However, as shown empirically in Appendix~\ref{sec:appendix:uncert}, $\vec{s}(\vx)$ estimated from random sampling is a poor measure of uncertainty, which is unsurprising for high dimensional spaces.}

\new{We instead derive a closed form for $\vec{m}(\vx), \vec{s}(\vx)$ based on the following intuition. The \cs{} estimated using $\csim(f(\vx), v_k)$ must intuitively be in the ballpark of $cos(\theta_k)=\csim(g(\vx), w_k)$ where $\csim$ is the cosine similarity~\citep{wiki:Cosine_similarity} (we switched from dot-products to $\csim$ to avoid differences due to magnitude of the vectors). However, if the concept $k$ is not encoded in $f(\vx)$ or if it is ambiguous, the \cs{} are expected to deviate by an angle $\alpha_k$, which is an error measure specific to the concept. Therefore, we expect the \cs{} to be in the range of $cos(\theta_k\pm \alpha_k)$. The concept specific value $\alpha_k$ must account for  uncertainty due to lack of knowledge (for eg. irrelevant concept) and due to ambiguity (for eg. concept is not clearly visible). In what follows, we present a specific measure for $\alpha_k$ and the closed form solution for $\vec{m}(\vx), \vec{s}(\vx)$. 
}


\new{Borrowing from \citet{Oikarinen2023}, we define $cos(\alpha_k)$ \newline as $\max_v [\csim(e(v, f, \mathcal{D}), e(w_k, g, \mathcal{D}))]$ where $e(w_k, g, \mathcal{D})\triangleq [w_k^Tg(\vx_1), \dots, w_k^Tg(\vx_N)]^T$, and \newline $e(v, f, \mathcal{D}) \triangleq [v^Tf^{[-1]}(\vx_1),\dots, v^Tf^{[-1]}(\vx_N)]^T$.\newline We may just as well adopt any other measure for $\alpha_k$.}
\begin{prop}
    For a concept k and $\alpha_k$ defined as above, we have 
    $$\vec{m}(\vx)_k=cos(\theta_k)cos(\alpha_k), \quad \vec{s}(\vx)_k=sin(\theta_k)sin(\alpha_k)$$
    where cos($\theta_k$)=$\csim(g_{text}(T_k), g(\vx))$ and $\vec{m}(\vx)_k, \vec{s}(\vx)_k$ denote the $k^{th}$ element of the vector.
    \label{prop:1}
\end{prop}
The proof can be found in Appendix~\ref{sec:appendix:proof1}. The mean and scale values above have a clean interpretation. If the \mex{} ($f$) uses the $k^{th}$ concept for label prediction, the information about the concept is encoded in $f$ and we get a good fit, i.e. $cos(\alpha_k)\approx 1$, and a small error on \cs{}. On the other hand, error bounds are large and \cs{} are suppressed when the fit is poor, i.e. $cos(\alpha_k)\approx 0$. 

Although $s(\vx)$ given by Proposition~\ref{prop:1} is simple, we found it to be surprisingly effective, which are presented in Appendix~\ref{sec:appendix:uncert}. The appendix section contains the following details. (1) We introduced two other variants for measuring uncertainty: Monte Carlo and a distribution fit method (based on~\citep{KimJung23}). (2) We evaluated different methods for quantifying the different sources of uncertainty on a real dataset. Across different sources of uncertainty and uncertainty estimation methods, we found our method to be consistently effective.  


\subsection{Theoretical motivation}
\label{sec:method:analysis}
We now demonstrate theoretically the pitfalls of using a standard method for estimating concept explanations. For ease of analysis, we focus on robustness to misspecified concept sets. In our study, we compared explanations generated using a standard linear estimator (using Ordinary Least Squares~\citep{wiki:OLS}) and \ours{}. Recall that Posthoc-CBMs (O-CBM, Y-CBM), which are our primary focus for comparison, both estimate explanations by fitting a linear model on \cs{}. 

We present two scenarios with noisy \cs{}. 
In the first scenario (over-complete concept set), we analysed the estimation when the concept set contains many irrelevant concepts. We show that the likelihood of marking an irrelevant concept as more important than a relevant concept increases rapidly with the number of concepts when the explanations are estimated using a standard linear estimator, which is presented in Corollary~\ref{cor:prop2}. In the same result, we also demonstrated that \ours{} do not suffer the same problem.  In the second scenario (under-complete concept set), we analysed the explanations when the concept set only includes irrelevant concepts, wherein both concepts should be assigned importance at or near zero. We again show in Proposition~\ref{prop:3} that a standard linear model attributes a significantly non-zero score while \ours{} mitigates the issue. 

{\bf Setting 1: Unreliable explanations due to over-complete concept set}. 
We analyse a simple setting where the output (y) is linearly predicted from the input ($\vx$) as $y=\vw^T\vx$. We wish to estimate the importance of some K concepts by fitting  a linear estimator on \cs{}. Where \cs{} are computed as $\vw_k^T\vx$ using concept activation vectors ($\vw_k$) that are distributed as $\vw_k\sim \mathcal{N}(\vu_k, \sigma_k^2I), k\in [1, K]$.
\begin{prop}
  The concept importance estimated by \ours{} when the input dimension is sufficiently large and for a regularizing hyperparameter $\lambda>0$ is approximately given by $v_k=\frac{\vu_k^T\vw}{\vu_i^T\vu_k + \lambda\sigma_k^2}$. On the other hand, the importance scores estimated using a standard estimator under the same conditions is distributed as $v_k\sim \mathcal{N}(\frac{\vu_k^T\vw}{\vu_k^T\vu_k}, \sigma_k^2\frac{\|\vw\|^2}{\|\vu_k\|^2})$.
  \label{prop:2}
\end{prop}

Proof of the result can be found in Appendix~\ref{sec:appendix:proof2}. \new{Based on the result, we can deduce the following for a specific case of $\vu_k$s and $\sigma_k$s.   
\begin{cor}
    For the data setup of Proposition~\ref{prop:2}, the following results holds when $\vu_1=\mathbf{w}, \sigma_1\approx 0$ and $\vu_k^T\mathbf{w}=0,\quad \forall k \in [2, K]$. Then the probability that the standard estimator returns the first concept as the most salient decreases exponentially with the number of concepts. On the other hand, the importance score assigned by \ours{} is 1 for the only relevant first concept and 0 otherwise. 
    \label{cor:prop2}
\end{cor}
Derivation of the result can be found in Appendix~\ref{sec:cor:2}.} We observe from the result that the standard estimator will more likely flag an irrelevant concept as relevant, which is addressed when using \ours{}. 
Section~\ref{sec:expt:stl} confirms our theoretical observation in practice.  

{\bf Setting 2: Unreliable explanations due to under-complete concept sets}.
We now analyse explanations when the concept set only includes two irrelevant concepts. Like in the previous setting, the $k^{th}$ concept activation $c_k^{(i)}$ is computed as $\vw_k^T\vx^{(i)}$. Consider two orthogonal unit vectors $\vu, \vv$. We define the concept activations for the two concepts $c_1^{(i)}, c_2^{(i)}$ for the $i^{th}$ instance $\vx^{(i)}$ and label $y^{(i)}$ as below.
\begin{align*}
y^{(i)} = \vu^T\vx^{(i)}, \quad &c_{1}^{(1)}=(\beta_1\vu + (1-\beta_1)\vv)^T\vx^{(i)}\\ 
&c_{2}^{(i)} = (\beta_2\vu + (1-\beta_2)\vv)^T\vx^{(i)}\\
\text{where } & \beta_1 \sim \mathcal{N}(b_1, \sigma^2), \beta_2 \sim \mathcal{N}(b_2, \sigma^2)
\end{align*}
If $b_1, b_2, \sigma^2$ are very small, then both concepts are expected to be unimportant for label prediction due to small overlap with $\vu$. However, we can see with some effort (Appendix~\ref{sec:appendix:proof3}) that the importance scores computed by a standard estimator are   $\frac{1-\beta_2}{\beta_1-\beta_2}, \frac{1-\beta_1}{\beta_1 - \beta_2}$, which are large because $\beta_1\approx 0, \beta_2\approx 0 \therefore \beta_1-\beta_2\approx 0$. We will now show that \ours{} estimates near-zero importance scores as expected. 
\begin{prop}
  The importance score estimated by \ours{} is approximately $\frac{b_1/2}{1+\lambda\sigma^2}, \frac{b_2/2}{1+ \lambda\sigma^2}$, where $\lambda>0$ is a regularizing hyperparameter. 
\label{prop:3}
\end{prop}
Proof can be found in Appendix~\ref{sec:appendix:proof3}. It follows from the result that the importance scores computed by \ours{} are near-zero for sufficiently large value of $\lambda$.

\section{Experiments}
\label{sec:expt}
\begin{figure*}[htb]
    \begin{minipage}{0.2\textwidth}
        \includegraphics[width=0.96\linewidth]{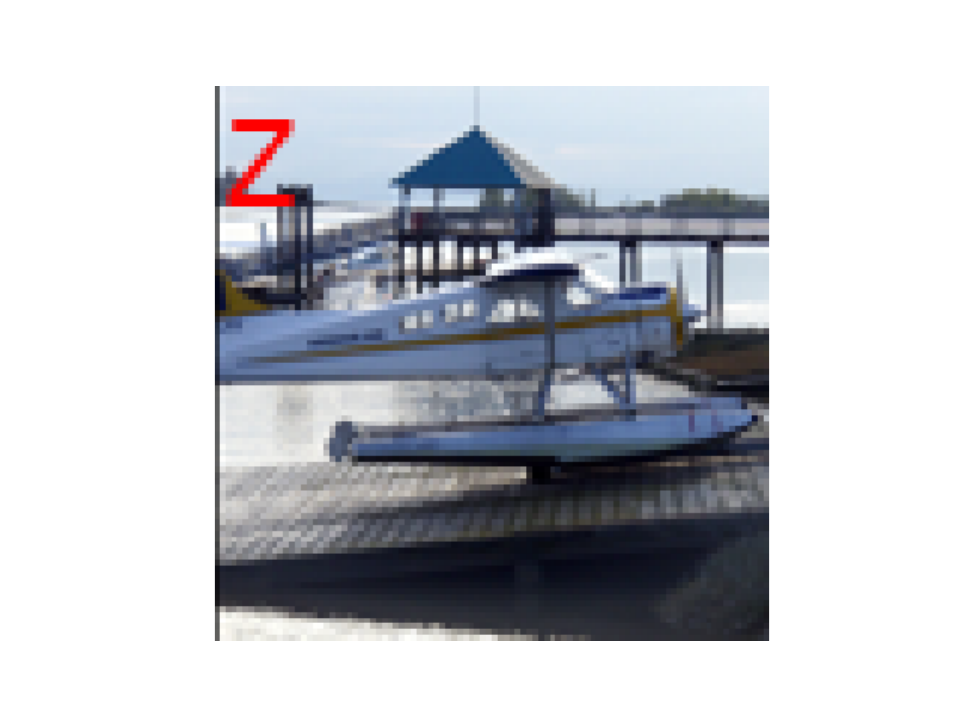}\\
        \includegraphics[width=0.96\linewidth]{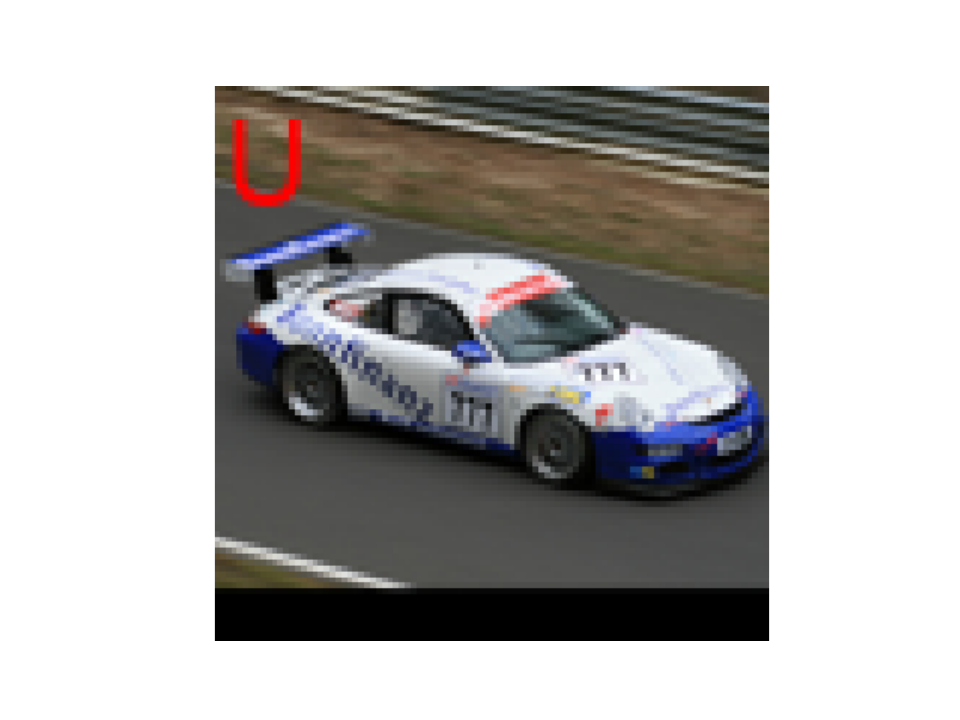}        
    \end{minipage}\hfill
    \begin{minipage}{0.4\textwidth}
        \includegraphics[width=0.98\linewidth]{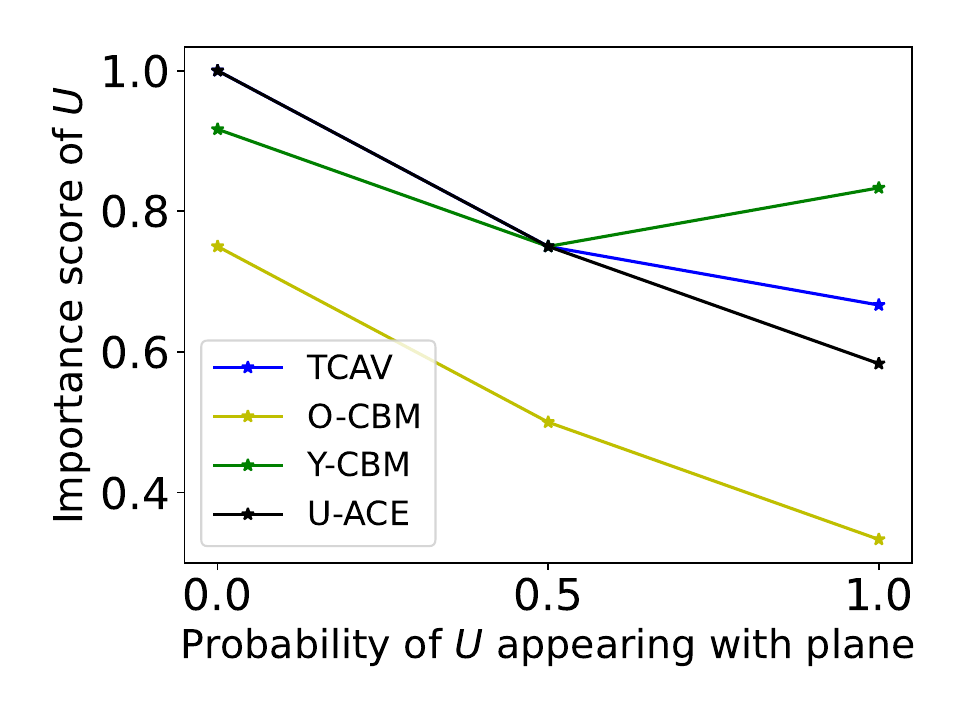}
    \end{minipage}\hfill
    \begin{minipage}{0.4\textwidth}
        \includegraphics[width=0.98\linewidth]{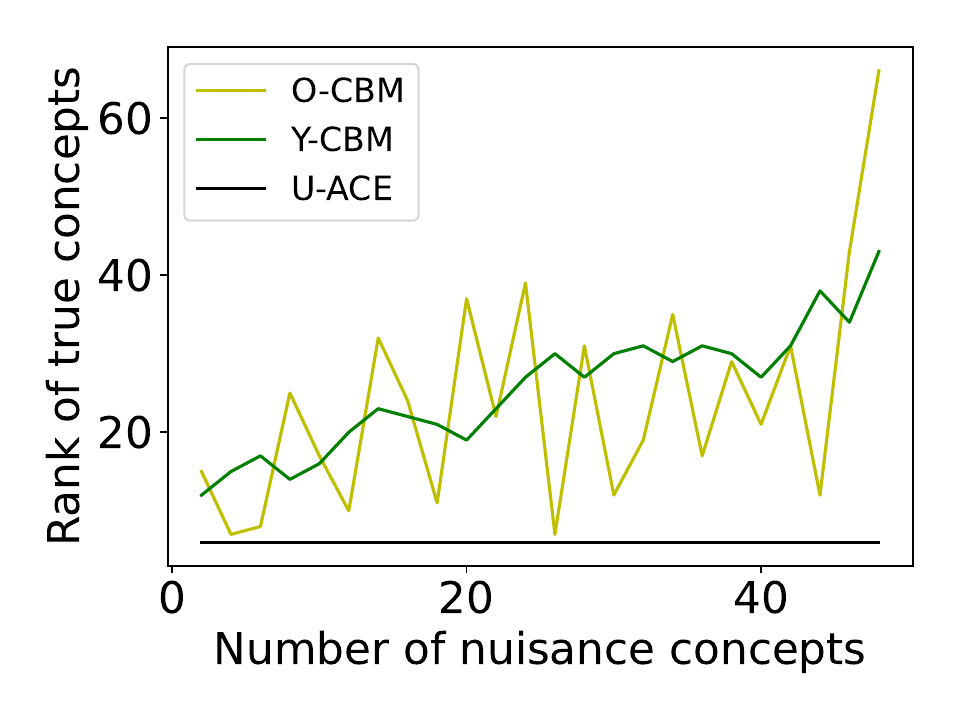}
    \end{minipage}
    \caption{(Left) STL dataset with a spurious tag. (Middle) Importance of a tag concept for three different \mex{}. X-axis shows the probability of tag in the training dataset of \mex{}. (Right) Average rank of true concepts with irrelevant concepts; lower is better.}
    \label{fig:stl}
\end{figure*}

We now empirically evaluate \ours{} on synthetic and real datasets. 
We make a quantitative assessment with known ground-truth concepts on a controlled dataset in Section~\ref{sec:expt:stl}. In Section~\ref{sec:expt:real}, we evaluate on two challenging real-world datasets with more than 700 concepts. Finally in Section~\ref{expt:simagenet}, we evaluate utility of different methods in recovering known spurious features from more than 2,000 concepts. 

\paragraph{Baselines.} \simple{}: Explanations are estimated using lasso regression of ground-truth concept annotations to estimate logit values of $f$. \simple{} was also adopted in the past~\citep{Ramaswamy2022elude, Ramaswamy2022}. Other baselines are introduced in Section~\ref{sec:background}: {\it TCAV}~\citep{TCAV}, {\it O-CBM}~\citep{Oikarinen2023}, {\it Y-CBM} based on~\citep{PosthocCBMICLR2023}.

\paragraph{Standardized comparison between importance scores.} The interpretation of the importance score varies between different estimation methods. For instance, the importance score in TCAV is the fraction of examples that meet certain criteria, but for the rest the importance scores are the weights from linear model that predicts logits. Further, \simple{} operates on binary concept annotations  and {\it O-CBM, Y-CBM, \ours{}} on soft scores. For this reason, we cannot directly compare importance scores or their normalized variants. We instead use negative scores to obtain a ranked list of concepts and assign to each concept an importance score given by its rank in the list normalized by number of concepts. Our sorting algorithm ranks any two concepts with same score by alphabetical order of their text description. In all our comparisons we use the rank score if not mentioned otherwise.  

\paragraph{Other experiment details.} For all our experiments, we employed a Visual Transformer (with 32 patch size called ``ViT-B/32'') based pretrained CLIP model that is publicly available for download at \url{https://github.com/openai/CLIP} as the pretrained multimodal model, which we denoted by $g$. We use $l=-1$, i.e. last layer just before computation of logits for all the explanation methods. \ours{} returns the mean and variance of the importance scores as shown in Algorithm~\ref{alg:ours}, we use mean divided by standard deviation as the importance score estimated by \ours{} everywhere for comparison with other methods. 

\section{Assessment with known ground-truth}
\label{sec:expt:stl}
We now seek to establish that \ours{} generates faithful concept explanations. Subscribing to the common evaluation practice~\citep{TCAV}, we generate explanations for a model that is trained on a dataset with controlled correlation of a spurious pattern. We make a dataset using two labels from STL-10 dataset~\citep{stl} {\it car, plane} and paste a tag {\it U} or {\it Z} in the top-left corner as shown in the left panel of Figure~\ref{fig:stl}. The probability that the examples of {\it car} are added the {\it Z} tag is p and 1-p for the {\it U} tag. Conversely for the examples of {\it plane}, the probability of {\it U} is p and {\it Z} is 1-p. We generate three training datasets with p=0, p=0.5 and p=1, and train three classification models using 2-layer convolutional network. As such, the three models are expected to have a varying and known correlation with the tag, which we hope to recover from its concept explanation.    


We generate concept explanations for the three \mex{} using a concept set that includes seven car-related concepts
and three plane-related concepts (Appendix~\ref{sec:appendix:misc})
along with the two tags {\it U, Z}.   
We obtain the importance score of the concept {\it U} with {\it car} class using a probe-dataset that is held-out from the corresponding training dataset (i.e. probe-dataset has the same input distribution as the training dataset). The results are shown in the middle plot of Figure~\ref{fig:stl}. Since the co-occurrence probability of $U$ with {\it car} class goes from 1, 0.5 to 0 for p=0, 0.5, 1, we expect the importance score of $U$ should change from positive to negative as we move right. We note that \ours{}, along with others, show the expected decreasing importance of the tag concept. 
The result corroborates that \ours{}, along with others, estimate a faithful explanation in the standard evaluation setting. Next, we evaluate robustness to misspecified or overly-complete concept set. 

\paragraph{Unreliability due to a over-complete concept set.}
We generate explanations as animal (irrelevent and therefore nuisance) concepts are added (Appendix~\ref{sec:appendix:misc} contains the full list) to the relevant list of twelve original concepts. The right panel of Figure~\ref{fig:stl} depicts the average rank of true concepts (lower the better) with the number of irrelevant concepts on the horizontal axis. We note that \ours{} ranks true concepts highly even with 50 nuisance concepts while O-CBM and Y-CBM increasingly get worse with the number of irrelevant concepts as predicted by Corollary~\ref{cor:prop2}. 



\section{Real-world evaluation}
We expect the improved modeling of our estimator to also generate higher-quality concept explanations in practice. To verify, we next explore explanations for two off-the-shelf pretrained models. 
\subsection{Scene Classification}
\label{sec:expt:real}
\begin{figure*}[t]
  \begin{minipage}{0.10\textwidth}
    \begin{figure}[H]
      \includegraphics[width=0.98\textwidth]{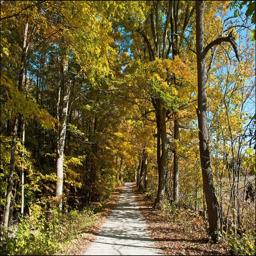}
    \end{figure}
  \end{minipage} \hfill
  \begin{minipage}{0.38\textwidth}
    \textbf{Tree Farm}\newline
    \simple{}: {\texttt{tree, field, bush}} \newline
    O-CBM: \texttt{forest, pot, {\color{red} sweater}}\newline
    Y-CBM: {\small \texttt{field, forest, {\color{red}elevator}}}\newline
    {\it \ours{}}: \texttt{foliage, forest, grass}
  \end{minipage}
  \begin{minipage}{0.10\textwidth}
    \begin{figure}[H]
      \includegraphics[width=0.99\textwidth]{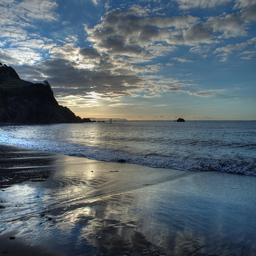}
    \end{figure}
  \end{minipage} \hfill
  \begin{minipage}{0.38\textwidth}
    \textbf{Coast}\newline
    \simple{}: \texttt{sea, water, river}\newline
{\it O-CBM}: \texttt{sea, island, {\color{red} pitted}}\newline
{\it Y-CBM}: \texttt{sea, sand, {\color{red} towel rack}}\newline
{\it \ours{}}: \texttt{sea, lake, island}
  \end{minipage}
    \begin{minipage}{0.10\textwidth}
    \begin{figure}[H]
      \includegraphics[width=0.99\textwidth]{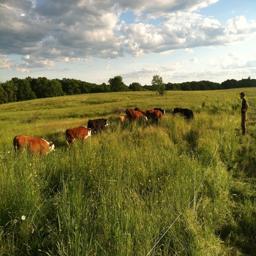}
    \end{figure}
    \end{minipage} \hfill
  \begin{minipage}{0.38\textwidth}
  \vspace{2mm}
    \textbf{Pasture}\newline
    \simple{}: \texttt{horse, sheep, grass}\newline
{\it O-CBM}: \texttt{shaft, hoof, {\color{red} exhibitor}}\newline
{\it Y-CBM}: \texttt{field, grass, {\color{red} ear}}\newline 
{\it \ours{}}: \texttt{grass, cow, {\color{red} banded}}\newline 
  \end{minipage}
    \begin{minipage}{0.10\textwidth}
    \begin{figure}[H]
      \includegraphics[width=0.99\textwidth]{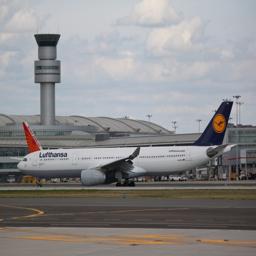}
    \end{figure}
    \end{minipage} \hfill
  \begin{minipage}{0.38\textwidth}
    \textbf{Runway}\newline
    \simple{}: \texttt{plane, field, sky}\newline
    {\it O-CBM}: \texttt{plane, fuselage, {\color{red} apron}}\newline
    {\it Y-CBM}: {\small \texttt{plane, clouds, {\color{red} candlestick}}}\newline
    \ours{}: \texttt{plane, windscreen, sky}
  \end{minipage}
  \caption{Two most relevant concepts plus any mistake (marked in red) from top-10 concepts for a scene-classification model estimated with various algorithms using PASCAL (left) or ADE20K (right) probe-dataset.}
  \label{fig:exptE}
  \vspace{-3mm}
\end{figure*}

In this section, we evaluate concept explanations obtained for a scene classification model with ResNet-18 architecture pretrained on the publicly available Places365~\citep{places365}. Following the experimental setting of \citet{Ramaswamy2022}, we generate explanations when the probe-dataset is set to PASCAL~\citep{Pascal} or ADE20K~\citep{ADE}, which are both part of the Broden dataset~\citep{NDissect}. The dataset contains images with dense annotations with more than 1000 attributes. Since it is irrational to explain a scene using scene concepts, we drop around 300 concepts that are marked as scene-related in the dataset. 
For the remaining 730 attributes, we defined a concept per attribute using literal name of the attribute. We picked 50 scene labels (Appendix~\ref{sec:appendix:misc} contains the full list) that have support of at least 20 examples in both ADE20K and PASCAL datasets. 

We evaluate the quality of explanations by their closeness to the explanations generated using the \simple{} baseline. Because \simple{} fits an estimator using human-annotated concept labels, they are the closest to the ground-truth. For the top-20 concepts identified by \simple{}, we compute the average absolute difference in importance scores estimated using any estimation method and \simple{}. Table~\ref{tab:real:expls} presents the deviation in explanations averaged over all the 50 scene labels. Figure~\ref{fig:exptE} shows the most salient concepts for four randomly picked scene labels. We observe from the figure that top-10 concepts identified by \ours{} seem more relevant to the scene when compared with Y-CBM and O-CBM. We also evaluated the explanation quality using a standard measure for comparing ranked lists, which is presented in Appendix~\ref{sec:appendix:misc}, which further confirms the dominance of \ours{}. 

{\bf Dataset shift.}
\citet{Ramaswamy2022} demonstrated with results the drastic shift in concept explanations for the same \mex{} when using ADE20K or PASCAL as the probe-dataset. Explanations diverge partly because (a) population of concepts may vary between datasets thereby influencing their perceived importance when using standard methods, (b) noise in estimated explanations.  
We have demonstrated that \ours{} estimated importance scores have low noise and attributes high uncertainty and thereby near-zero importance to concepts that are rare or missing from the probe-dataset (Section~\ref{sec:simstudy}). For these reasons, we expect \ours{} to mitigate the data-shift problem. We confirm the same by estimating the average difference in importance scores estimated using ADE20K and PASCAL for different estimation techniques (where the average is only over salient concepts with non-zero importance). The results are shown in Table~\ref{tab:real:shift} and are inline with our prediction. 

\begin{table}[htb]
\centering
  \setlength{\tabcolsep}{3pt} 
    \begin{tabular}{c|r|r|r|r}
    Dataset$\downarrow$ & TCAV & O-CBM & Y-CBM & \ours{}\\\hline
       ADE20K  & 0.13 & 0.19 & 0.16 & {\bf 0.09} \\
       PASCAL  & 0.41 & 0.20 & 0.18 & {\bf 0.11}
    \end{tabular}
    \captionof{table}{{\it Evaluation of explanation quality.} Each cell shows the average absolute difference of importance scores for top-20 concepts estimated using \simple{}.}
    \label{tab:real:expls}
\end{table}
\begin{table}[htb]
\centering
  \setlength{\tabcolsep}{3pt} 
    \begin{tabular}{r|r|r|r|r}
      \simple{} & TCAV & O-CBM & Y-CBM & \ours{}\\\hline
      0.41 & 0.41 & 0.32 & 0.33 & {\bf 0.19}
    \end{tabular}
    \captionof{table}{{\it Effect of data shift.} Average absolute difference between concept importance scores estimated using ADE20K and PASCAL datasets for the same \mex{} using different estimation methods.}
    \label{tab:real:shift}
\end{table}

\subsection{Object Classification}
\label{expt:simagenet}
In this section, we evaluate explanations obtained for an object classification model with ResNet-18 architecture that was pretrained on the Imagenet dataset. Unlike for Broden dataset, Imagenet does not have annotations over concepts. As such, we evaluate the quality of explanations according to how effectively they bring out previously known incidental or spurious correlations of the dataset. Salient-Imagenet~\citep{Singla2022} is an impressive annotation effort to identify prevalent and problematic dependence on co-occurring but irrelevant features for object recognition such as {\it door} feature when classifying {\it doormat}. We are interested in examining if we can discover such incidental features using concept explanations. We gathered 2,020 words by tokenising and filtering notes left by annotators of Salient-Imagenet, which we use as the concept set. One may also populate the concept set by generating a caption for each image using the pretraied multimodal model. For the sake of evaluation, we randomly identified 40 classes with at least one known incidental feature, and manually lexicalized the concept corresponding to the incidental feature for each class (see Appendix~\ref{appendix:simagenet} contains the full list). 

In Table~\ref{tab:simagenet_discovery}, we show the success of different estimation methods in discovering a known incidental correlation. We observe that \ours{} is the most effective in capturing the dependence. Although O-CBM and Y-CBM seem effective, we observed that they identified many nonsensical concepts (items that are very unlikely to be relevant) as relevant as shown in Table~\ref{tab:simagenet_example}. 

We are glad that \ours{} discovered known spurious features for 36 of the 40 labels (Table~\ref{tab:simagenet_discovery}) with little effort while also being accurate (Table~\ref{tab:simagenet_example}). However, our experiment only quantified recall of the spurious feature with only anecdotal evidence for precision. We leave for a future study a more tighter evaluation quantifying both recall and precision along with user studies to measure the utility of \ours{} for model debugging.  
\begin{table}[htb]
    \centering
    \begin{tabular}{l|r|r|r}
        k& O-CBM & Y-CBM & \ours{}\\\hline
        10 &  14 & 22 & {\bf 25}\\
        25 &  23 & 28 & {\bf 33}\\
        50 & 31 & 29 & {\bf 36}
    \end{tabular}
    \caption{Number of classes (of the total 40) for which the known spurious feature is found in the top-k concepts estimated using different methods shown in the top-row.}
    \label{tab:simagenet_discovery}
\end{table}

\begin{table}[htb]
    \centering
    \begin{tabular}{l|c|l}
    \setlength{\tabcolsep}{2pt}
         howler & Y-CBM & \texttt{underside, monkey, \color{red}{cat}}\\
         \hspace{3mm} monkey& O-CBM & \texttt{\color{red}{polecat, porcupine, bear}}\\
         & \ours{} & \texttt{siamang, colobus, gibbon}\\\hline
        coast & Y-CBM & \texttt{shore, \color{red}{view, maillot}} \\
         & O-CBM & \texttt{beach, shore, \color{red}{tyre}} \\
         & \ours{} & \texttt{beach, sea, boats} \\\hline
    \end{tabular}
    \caption{Top-2 concepts plus any mistake identified from top-10 important concepts shown in red. More results in Table~\ref{tab:simagenet_examples:more}}
    \label{tab:simagenet_example}
\end{table}


\section{Related Work}

{\bf Concept Bottleneck Models} use a set of predefined human-interpretable concepts as an intermediate feature representation to make the predictions ~\citep{Koh2020,Bau2017,TCAV,Zhou2018}. CBM allows human test-time intervention which has been shown to improve overall accuracy ~\citep{barker2023selective}. Traditionally, they require labelled data with concept annotations and accuracy is typically worse than the standard models without concept bottleneck. To address the limitation of concept annotation, recent works have leveraged large pretrained multimodal models like CLIP ~\citep{Oikarinen2023, PosthocCBMICLR2023}. 
Concept Embedding Models (CEM) ~\citep{espinosa2022concept} overcome the trade-off between accuracy and interpretability by learning high-dimensional concept embeddings. However, addressing the noise in the concept prediction remains underexplored. \citet{collins2023human} have studied \textit{human} uncertainty in concept-based models and elucidate the importance of considering uncertainty over concepts in improving the reliability of the model. Closely related to our work, \citet{KimJung23} propose the Probabilistic Concept Bottleneck Models (ProbCBM). They too argue for the need to model uncertainty in concept prediction for reliable explanations. However, their method of noise estimation in \cs{} requires retraining the model and cannot be applied directly when \cs{} are estimated using CLIP. Moreover, they use simple MC sampling to account for noise in \cs{}, which is not nearly as effective (Appendix~\ref{sec:appendix:uncert}).

{\bf Concept-based explanations} use a separate probe dataset to first learn the concept and then explain through decomposition either the individual predictions or overall label features. \citet{Yeh2022} contains a brief summary of existing concept based explanation methods. Our proposed method is very similar to concept-based explanations (CBE)~\citep{TCAV,Bau2017,Zhou2018,Ghorbani19}. \citet{Ramaswamy2022} emphasised that the concepts learned are sensitive to the probe dataset used and therefore pose problems when transferring to applications that have distribution shift from the probe dataset. They further highlight the drawback of existing CBE methods that concepts can sometimes be harder to learn than the label itself (meaning the explanations may not be causal) and that the typical number of concepts used for explanations far exceed what a typical human can parse easily. \citet{Achtibat2022} championed an explanation method that provides explanation highlighting important feature (answering ``where'') and what concepts are used for prediction thereby combining the strengths of global and local explanation methods. \citet{Choi23} have built upon the current developments in CBE methods for providing explanations for out-of-distribution detectors. \citet{Wu23causal} introduced the causal concept based explanation method (Causal Proxy Model), that provides explanations for NLP models using counterfactual texts. \citet{Moayeri23} also used CLIP to interpret the representations of a different model trained on uni-modal data.


\section{Conclusion}
We studied concept explanation methods with a focus on data-efficient systems that exploit pretrained multimodal models. We highlighted the quality challenge of the of existing estimators of concept explanations via simple examples and motivated the need for modelling \textit{uncertainty} in their estimation. Accordingly, we proposed an uncertainty-aware and data-efficient estimator called \ours{}. We demonstrated the merits of our estimator through theoretical analysis, controlled study experiments and three challenging real-world evaluation with order of thousand concepts. Our results establish the strong promise of concept explanations estimated using our Bayesian method for effective model debugging. \\
{\bf Limitations and Future Work}. Pretrained multimodal models enabled us to work with an open-world specification of the set of concepts, without necessitating expensive annotation data. Yet, dependence on a pretrained model may limit application in certain specialised domains, which while emerging in domains like healthcare ~\citep{huang2023visual}, may temporarily hinder widespread adoption. In the same vein, an estimator that also models the uncertainty due to lack of knowledge of the pretrained multimodal model about a concept will be better suited for addressing any lapses in the pretrained model, which we leave for the future. Additionally, while we do not run user studies in this work, a natural and exciting next step is to explore whether \ours{} explanations are preferred by real humans. 

\newpage
\bibliography{references}  

\newpage
\appendix
\onecolumn

\title{Estimation of Concept Explanations Should be Uncertainty-Aware\\(Supplementary Material)}
\maketitle

\section{Miscellaneous}
\subsection{Derivation of posterior on weights}
\label{sec:appendix:posterior}
The result of posterior distribution of weights follows directly from the form of posterior under normal prior on weights as explained as~\citet{SMLRuss} (Slide 10). For the sake of completeness, we also derive the result below. 
\begin{align*}
    \Pr(\vec{w}\mid C_X, Y) &\propto \Pr(Y\mid C_X, \vec{w})\Pr(\vec{w})\\
    &=\mathcal{N}(Y; C_X^T\vec{w}, \beta^{-1})\mathcal{N}(\vec{w}; 0, S_0)\text{ where } S_0^{-1}=\lambda^{-1}\text{diag}(\epsilon\epsilon^T)\\
    &\propto \exp\left\{ -\frac{\beta}{2} (Y-C_X^T\vec{w})^T(Y-C_X^T\vec{w}) - \frac{1}{2}\vec{w}^TS_0^{-1}\vec{w} \right\}\\
    &\propto \exp\left\{ -\frac{1}{2}\vec{w}^T[\beta C_XC_X^T + S_0^{-1}]\vec{w} - \beta(C_XY)\vec{w} \right\}
\end{align*}
We see that the posterior also takes the form of normal distribution with $\Sigma^{-1} = \beta C_XC_X^T + S_0^{-1}$ and $\mu = \beta\Sigma C_XY$.

\subsection{Algorithm}
We describe the algorithm summarizing \ours{} in~\ref{alg:ours}. An additional technical detail of the algorithm is a step to sparsify weights as described below. \\
{\bf Sparsifying weights for interpretability.}
As a dense weight matrix can be hard to interpret, we induce sparsity in $W_c$ by setting all the values below a threshold to zero. We pick the threshold such that the accuracy on train split does not fall by more than $\kappa$, which is a positive hyperparameter.

\begin{algorithm}
  \caption{\oursfull{} (\ours)}\label{alg:ours}
\begin{algorithmic}
  \Require $\mathcal{D}$=$\{\vx_1, \vx_2, \dots, \vx_N\}$, $\mathcal{T}=\{T_1, T_2, \dots, T_K\}$, f (\mex{}), g (CLIP), $\kappa$
  \For{$y=1,\dots, L$}
  \State Y = $[f(\vx)[y]$ for $\vx\in \mathcal{D}^T]$ \Comment{Gather logits}
  \State $C_X = [\vec{m}(\vx_1), \dots, \vec{m}(\vx_N)]$, $\epsilon$ = $\mathbb{E}_{\mathcal{D}}[\vec{s}(\vx)]$ \Comment{Estimate $\vec{m}(\vx), \vec{s}(\vx)$ (Section~\ref{sec:method:noise})}
  \State $\vec{w}_y\sim \mathcal{N}(\mu_y, \Sigma_y)$ where $\mu_y, \Sigma_y$ from Equation~\ref{eq:posterior} 
  \Comment{Estimate $\lambda, \beta$ using MLL }
  \EndFor
  \State $W_{c}$ = sparsify($[\vec{\mu}_1, \vec{\mu}_2, \dots \vec{\mu}_L]$, $\kappa$)\Comment{Suppress less useful weights, Section~\ref{sec:method}}
  \State \Return $W_{c}, [\text{diag}(\Sigma_1), \text{diag}(\Sigma_2), \dots \text{diag}(\Sigma_L)]$
\end{algorithmic}
\end{algorithm}

\section{Maximum Likelihood Estimation of \ours{} parameters}
\label{sec:appendix:mle}
The posterior on weights shown in Equation~\ref{eq:posterior} has two parameters: $\lambda, \beta$ as shown below with $C_X$ and Y are array of concept activations and logit scores (see Algorithm~\ref{alg:ours}). 
\begin{align*}
\vec{w}\sim \mathcal{N}(\mu, \Sigma)\qquad \text{ where } \mu = \beta\Sigma C_XY, \quad \Sigma^{-1} = \beta C_XC_X^T + \lambda^{-1} diag(\epsilon\epsilon^T)
\end{align*}
We obtain the best values of $\lambda$ and $\beta$ that maximize the log-likelihood objective shown below. 
\begin{align*}
    \lambda^*, \beta^* = \argmax_{\lambda, \beta} &\quad \mathbb{E}_Z[-\frac{\beta^2\|Y - (C_X + Z)^T\vec{w}(\lambda, \beta)\|^2}{2} + \log(\beta)]\\
    &\text{ where Z is uniformly distributed in the range given by error intervals}\\
    &Z\sim Unif([-\vec{s}(\vx_1), -\vec{s}(\vx_2), \dots, ], [\vec{s}(\vx_1), \vec{s}(\vx_2), \dots, ])
\end{align*}
We implement the objective using Pyro software library~\citep{pyro} and Adam optimizer. 

\section{Proof of Proposition~\ref{prop:1}}
\label{sec:appendix:proof1}
We restate the result for clarity. \\
For a concept k and $cos(\alpha_k)$ defined as $\csim(e(v_{k}, f, \mathcal{D}), e(w_{k}, g, \mathcal{D}))$, we have the following result when concept activations in $f$ for an instance $\vx$ are computed as $\csim(f(\vx), v_k)$ instead of $v_k^Tf(\vx)$. 
$$\vec{m}(\vx)_k=cos(\theta_k)cos(\alpha_k), \quad \vec{s}(\vx)_k=sin(\theta_k)sin(\alpha_k)$$
where cos($\theta_k$)=$\csim(g_{text}(T_k), g(\vx))$ and $\vec{m}(\vx)_k, \vec{s}(\vx)_k$ denote the $k^{th}$ element of the vector.
\begin{proof}
Corresponding to $v_k$ in $f$, we assume there is an equivalent vector $w$ in the embedding space of g such that $\csim(f(\vx), v_k)=\csim(g(\vx), w$) for any $\vx$. For example, the assumption is met when there is a linear mapping between $f(\vx)\approx Wg(\vx)$~\citep{Moayeri23} for a unitary matrix ($W^TW=I$), in which case w is simply $Wv_k$. For such a w, the following condition on $cos(\alpha_k)$ must hold as well. 
\begin{align*}
    cos(\alpha_k) = \csim(e(v_{k}, f, \mathcal{D}), e(w_{k}, g, \mathcal{D})) = \csim(e(w, g, \mathcal{D}), e(w_{k}, g, \mathcal{D}))
\end{align*}
Denote the matrix of vectors embedded using $g$ by $G=[g(\vx_1), g(\vx_2), \dots, G(\vx_N)]^T$ a $N\times D$ matrix (D is the dimension of $g$ embeddings). Let U be a matrix with S basis vectors of size $S\times D$. We can express each vector as a combination of basis vectors and therefore $G=AU$ for a $N\times S$ matrix A.  

Substituting the terms in the $\csim$ expression, we have:
\begin{align*}
 cos(\alpha_k) &= \csim(Gw, Gw_k) = \csim(AUw, AUw_k) \\
 &= \frac{w^TU^TA^TAUw_k}{\sqrt{(w^TU^TA^TAUw)(w_k^TU^TA^TAUw_k)}}. 
\end{align*}
If the examples in $\mathcal{D}$ are diversely distributed without any systematic bias, $A^TA$ is proportional to the identity matrix, meaning the basis of G and W are effectively the same. We therefore have $cos(\alpha_k) = \csim(Gw, Gw_k)=\csim(Uw, Uw_k)$, i.e. the projection of $w, w_k$ on the subspace spanned by the embeddings have $cos(\alpha_k)$ cosine similarity. Since $w, w_k$ are two vectors that are $\alpha_k$ apart, an arbitrary new example $\vx$ that is at an angle of $\theta$ from $w_k$ is at an angle of $\theta\pm \alpha_k$ from w. The cosine similarity follows as below. 
\begin{align*}
    cos(\theta) = \csim(w_k, g(\vx)) \implies \csim(w, g(\vx)) &= cos(\theta \pm \alpha_k) \\ &=cos(\theta)cos(\alpha_k) \pm sin(\theta)sin(\alpha_k)
\end{align*}
Because $w$ is a vector in $g$ corresponding to $v_k$ in $f$, $\csim(w, g(\vx)) = \csim(v_k, f(\vx))$.  
\end{proof}

\section{Proof of Proposition~\ref{prop:2}}
\label{sec:appendix:proof2}
The concept importance estimated by \ours{} when the input dimension is sufficiently large and for some $\lambda>0$ is approximately given by $v_k=\frac{\vu_k^T\vw}{\vu_k^T\vu_k + \lambda\sigma_k^2}$. On the other hand, the importance scores estimated using vanilla linear estimator under the same conditions is distributed as $v_k\sim \mathcal{N}(\frac{\vu_k^T\vw}{\vu_k^T\vu_k}, \sigma_k^2\frac{\|\vw\|^2}{\|\vu_k\|^2})$.
\begin{proof}
We use the known result that inner product of two random vectors is close to 0 when the number of dimensions is large, i.e. $\vu_i^T\vu_j\approx 0, i\neq j$. 

{\bf Solution with standard estimator.}
We first show the solution using vanilla estimator is distributed as given by the result above. 
We wish to estimate $v_1, v_2, \dots$ such that we approximate the prediction of \mex: $y=\vw^T\vx$. We denote by $\vw_k$ sampled from the normal distribution of concept vectors. We require $w^T\vx\approx \sum_k v_k\vw_k^T\vx$. In effect, we are optimising for $v$s such that $\|\vw - \sum_k v_k\vw_k\|^2$ is minimized. We multiply the objective by $\vu_k$ and use the result that random vectors are almost orthogonal in high-dimensions to arrive at objective $\argmin_{v_k} \|\vw_k^T\vw - v_k(\vw_k^T\vw_k)\|$. Which is minimized trivially when $v_k = \frac{\vw_k^T\vw}{\|\vw_k\|^2}$. 
Since $\vw_k$ is normally distributed with $\mathcal{N}(\vu_k, \sigma_k^2I)$, $\vw_k^T\vw = (\vu_k + \vec{\epsilon})^T\vw, \quad \vec\epsilon\sim \mathcal{N}(0, I)$ is also normally distributed with $\mathcal{N}(\vu_k^T\vw, \sigma_k^2\|\vw\|^2)$. We approximate the denominator with its average and ignoring its variance, i.e. $\|\vw_k\|^2 = \mathcal{N}(\|\vu_k\|^2, \sigma_k^2)\approx \|\vu_k\|^2$ which is when $\|\vu_k\|^2>>\sigma^2$. We therefore have the result that $v_k$ is normally distributed with mean $\frac{\vu_k^T\vw}{\|\vu_k\|^2}$ and variance $\sigma_k^2\frac{\|\vw\|^2}{\|\vu_k\|^2}$.

{\bf Solution with \ours{}.} 
Unlike the standard estimator, \ours{} seeks a solution by first estimating the distribution on \cs{}. We make an approximation of our proposed estimation of \ours{} and find a solution using the following objective, recall that the $k^{th}$ activation vector is defined to have been sampled from $\mathcal{N}(\vu_k, \sigma_k^2I)$.
$$
    \ell = \argmin_v \{\|\vw-\sum_k v_k\vu_k\|^2 + \lambda\sum_k \sigma_k^2 v_k^2\}
$$
With sufficient number of examples, the mean and variance $\vu_k, \sigma^2$ can be faithfully estimated. The above objective captures the essence of \ours{}(Eqn. \eqref{eq:posterior}) because (1) it fits the observations, and (2) regularizes importance of concepts with high variance. We now proceed with finding the closed form solution of $v_k$.
\begin{align*}
    &\ell = \argmin_v \{\|\vw-\sum_k v_k\vu_k\|^2 + \lambda\sum_k \sigma_k^2 v_k^2\}\\
    &\text{setting }\frac{\partial \ell}{\partial v_k} = 0 \text{ we obtain}\\
    &-\vu_k^T(w - \sum_j v_j\vu_j) + \lambda \sigma_k^2 v_k = 0\\
    &\text{ and using almost zero inner product result stated above, i.e. $\vu_i^T\vu_j\approx 0, i\neq j$, we get}\\
    & \implies v_k = \frac{\vu_k^T\vw}{\|\vu_k\|^2 + \lambda\sigma_k^2}
\end{align*}
\end{proof}

\section{Corollary of Proposition~\ref{prop:2}}
\label{sec:cor:2}
Following the result of Proposition~\ref{prop:2}, we have the following result on the $v_k$ estimated by \ours{} and the standard linear estimator. 
\begin{cor}
    For the data setup of Proposition~\ref{prop:2}, the following results holds when $\vu_1=\vw, \sigma_1\approx 0$ and $\vu_k^T\vw=0,\quad \forall k \in [2, K]$. Then the probability that the standard estimator returns the first concept as the most salient decreases exponentially with the number of concepts. On the other hand, the importance score assigned by \ours{} is 1 for the only relevant first concept and 0 otherwise. 
\end{cor}
\begin{proof}
    Plugging in the values for the special case of $\vu_1=\vw, \sigma_1\approx 0$ and $\vu_k^T\vw=0, k\geq 2$ in the closed form solution from Proposition~\ref{prop:2}, we have the following results for the standard linear estimator and \ours{}.

    {\bf Solution of standard estimator}. $v_1 = 1$ and $v_k\sim \mathcal{N}(0, \sigma_k^2\frac{\|\vw\|^2}{\|\vu_k\|^2})$ for k$\geq 2$.\newline
    For the first concept to remain the most salient, rest of the K-1 concepts must have an importance score less than 1. Recall that the probability that a random variable z$\sim \mathcal{N}(\mu, \sigma^2)$ less than a value $z_0$ is $\Phi(\frac{z_0-\mu}{\sigma})$ where $\Phi$ is the Cumulative Distribution Function of a standard normal distribution. Therefore the probability that all the K-1 concepts having a value less than 1 is $\prod_{k=2}^K\Phi(\frac{1-0}{\sigma_k\|\vw\|/\|\vu_k\|})=\prod_{k=2}^K\Phi(\frac{\|\vu_k\|}{\sigma_k\|\vw\|})$. Since the probability is a product over K-1 quantities, it decreases exponentially with K.  
    
    {\bf Solution of \ours{}}. $v_1=1, v_2, v_3,\dots=0$ follows directly from plugging in the values in to result of the proposition.
    
\end{proof}
  
\section{Proof of Proposition~\ref{prop:3}}
\label{sec:appendix:proof3}
  The importance score estimated by \ours{} is approximately $\frac{b_1/2}{1+\lambda\sigma^2}, \frac{b_2/2}{1+ \lambda\sigma^2}$, where $\lambda>0$ is a regularizing hyperparameter. 
\begin{proof}
Since $\beta_1, \beta_2$ are normal distributed, the concept activation vectors $\beta_1\vu + (1-\beta_1)\vv$ and $\beta_2\vu + (1-\beta_2)\vv$ are also normally distributed. We derive their closed form below. Recall that $\beta_1\sim \mathcal{N}(b_1, \sigma^2), \beta_2\sim \mathcal{N}(b_2, \sigma^2)$ and that $\vu, \vv$ are orthogonal and unit vectors. 
\begin{align*}
    \mathbb{E}[\beta_1\vu + (1-\beta_1)\vv] &= b_1\vu + (1-b_1)\vv\\
    \mathbb{V}[\beta_1\vu + (1-\beta_1)\vv] &= \mathbb{E}[(\beta_1-b_1)\vu + (b_1-\beta_2)\vv)^2]\\
    & = \mathbb{E}[(\beta_1-b_1)^2] + \mathbb{E}[(\beta_2-b_2)^2]\\
    & = 2\sigma^2
\end{align*}

As argued of the objective to approximate \ours{} in Appendix~\ref{sec:appendix:proof2}, we may obtain the closed form solution for the concept scores $\eta_1, \eta_2$ by solving the following objective. 
\begin{align*}
    \ell = \|\vu - \eta_1(b_1\vu + (1-b_1)\vv) - \eta_2(b_2\vu + (1-b_2)\vv)\|^2 + 2\lambda\sigma^2(\eta_1^2+\eta_2^2)
\end{align*}
By setting $\partial \ell/\partial \eta_1=0$, $\partial \ell/\partial \eta_2=0$, and simplifying, we get the following. 
\begin{align*}
b_1 &= \eta_1(b_1^2 + (1-b_1)^2 + 2\lambda\sigma^2) + \eta_2 (b_1b_2 + (1-b_1)(1-b_2))\\
b_2 &= \eta_2(b_2^2 + (1-b_2)^2 + 2\lambda\sigma^2) + \eta_1 (b_1b_2 + (1-b_1)(1-b_2))
\end{align*}
Eliminating variables, and solving for $\eta_1$, we get
\begin{align*}
    \eta_1 &= \frac{2\lambda b_1\sigma^2 + (1-b_2)(b_1-b_2)}{(b_1-b_2)(b_1+b_2-2b_1b_2) + 2\lambda\sigma^2(2\lambda\sigma^2 + b_1^2 + b_2^2 + (1-b_1)^2+(1-b_2)^2)}
    &=\text{similarly for $\eta_2$}
\end{align*}
By substituting $b_1-b_2\approx 0$, the expressin can be simplified as 
\begin{align*}
    \eta_1 &= \frac{b_1}{2\lambda\sigma^2 + b_1^2 + b_2^2 + (1-b_1)^2+(1-b_2)^2}
\end{align*}
If we now use the assumption that $b_1, b_2\approx 0$, $\therefore 1-b_1, 1-b_2\approx 1$, we get the final form below. 
\begin{align*}
    \eta_1 &= \frac{b_1}{2(1+\lambda\sigma^2)}\\
    \eta_2 &= \frac{b_2}{2(1+\lambda\sigma^2)}
\end{align*}


\end{proof}

{\bf Solution of standard estimator.}\\
\begin{align*}
    &\text{When }c_{1}^{(i)}=(\beta_1u + (1-\beta_1)v)^Tz^{(i)}, \quad c_{2}^{(i)} = (\beta_2u + (1-\beta_2)v)^Tz^{(i)} \\
    &\text{we can derive the value of the label by their scaled difference as shown below}\\
  &\frac{(1 - \beta_2)c_1^{(i)} - (1-\beta_1)c_2^{(i)}}{(1-\beta_2)\beta_1 - (1-\beta_1)\beta_2} = \frac{(1 - \beta_2)c_1^{(i)} - (1-\beta_1)c_2^{(i)}}{\beta_1 - \beta_2} = u^Tz^{(i)} = y^{(i)}\\
  &\implies \frac{1-\beta_2}{\beta_1-\beta_2}c_1^{(i)} + \frac{1-\beta_1}{\beta_1 - \beta_2}c_2^{(i)} = y_i\\
  &\implies v_1 = \frac{1-\beta_2}{\beta_1-\beta_2}, v_2 = \frac{1-\beta_1}{\beta_1 - \beta_2}
\end{align*}

\section{Additional experiment details}
\label{sec:appendix:misc}

\subsection{Simulated Study}
\label{sec:simstudy}
In this section, we consider explaining a two-layer CNN model trained to classify between solid color images with pixel noise as shown in Figure~\ref{fig:four_color}. The colors on the left (red; green) are defined as label 0 and the colors on the right (blue; white) are defined as label 1. 
\begin{wrapfigure}{r}{0.15\textwidth}
    \includegraphics[width=0.99\linewidth]{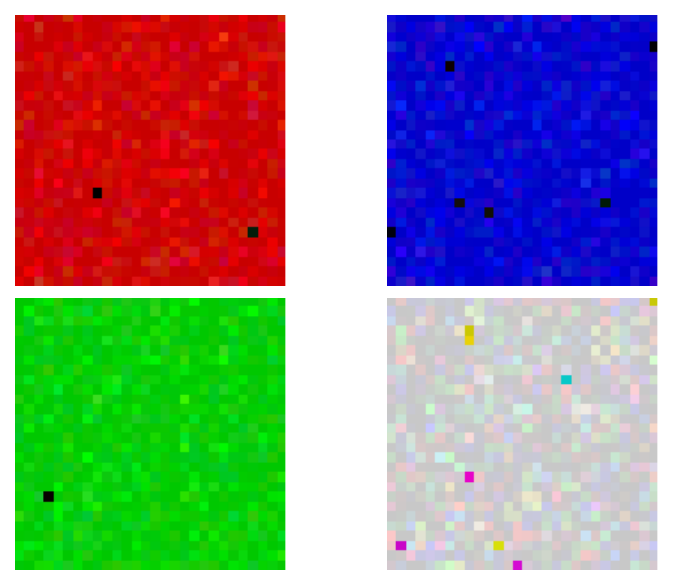}
    \caption{Toy}
    \label{fig:four_color}
    \vspace{-3mm}
  \end{wrapfigure}
The \mex{} is trained on a dataset with equal proportion of all colors; we therefore expect that all constituent colors of a label are equally important for the label. We specify a concept set with the four colors encoded by their literal name {\it red, green, blue, white}. \ours{} (along with others) attribute positive importance for {\it red, green} and negative or zero importance for {\it blue, white} when explaining label 0 using a concept set with only the four task-relevant concepts and when the probe-dataset is the same distribution as the the training dataset. However, quality of explanations quickly degrades when the probe-dataset is shifted, or if the concept set is misspecified. 

\begin{figure}[htb]
    \centering
    \includegraphics[width=0.32\linewidth]{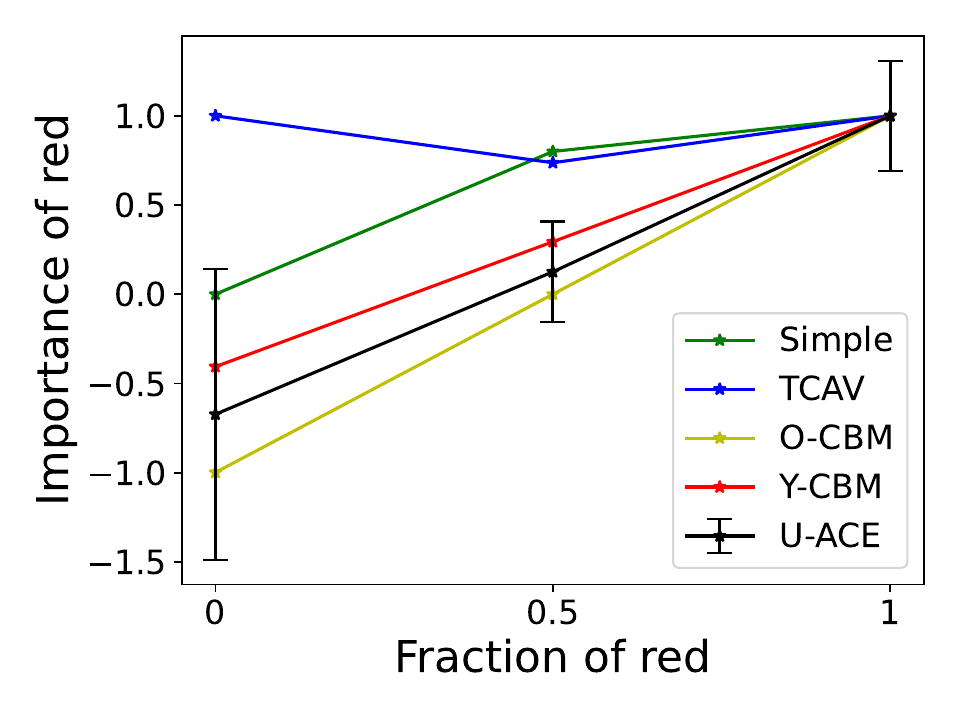}
    \includegraphics[width=0.32\linewidth]{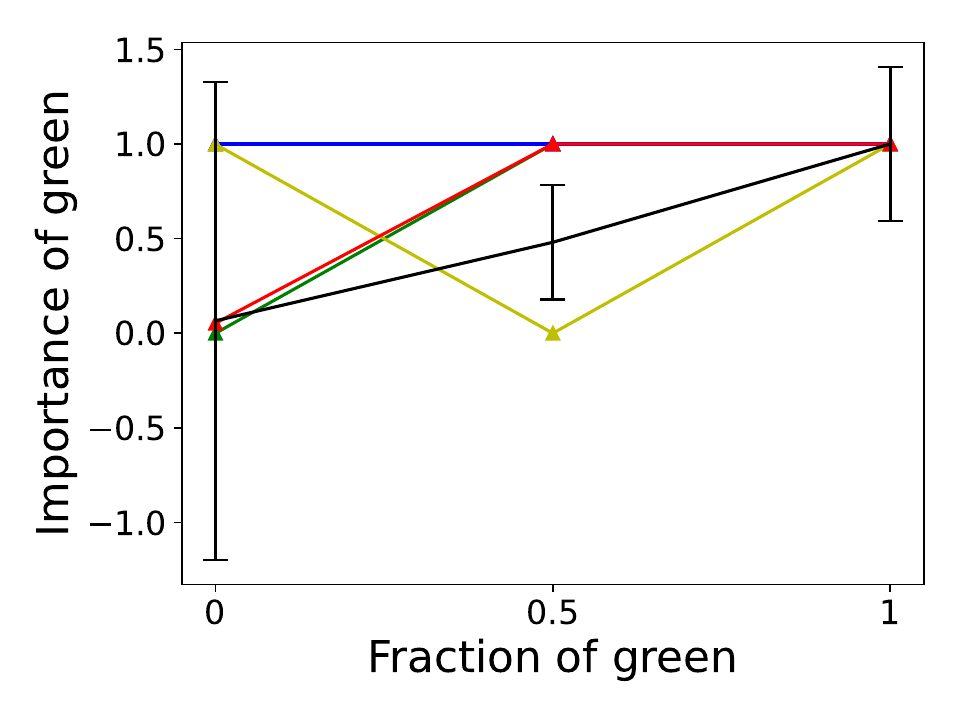}
    \includegraphics[width=0.32\linewidth]{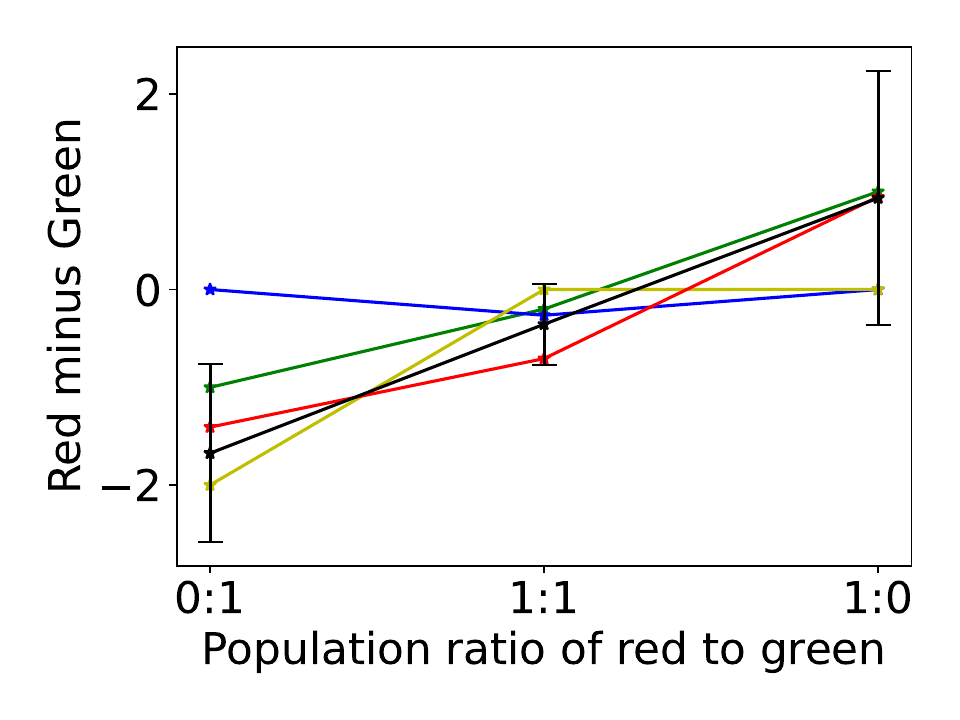}
     \caption{Left, middle plots show the importance of red and green concepts while the rightmost plot shows their importance score difference. \ours{} estimated large uncertainty in importance score when red or green concept is missing from the dataset as seen in the left of the left and middle plots. Also the difference in importance at either extreme in the right plot is not statistically significant. }
    \label{fig:4-color-expls}
\end{figure}

\paragraph{Unreliability due to dataset shift.} We varied the probe-dataset to include varying populations of colors while keeping the concept set and \mex{} fixed. We observed that importance of a concept estimated with standard CBEs varied with the choice of probe-dataset for the same underlying \mex{} as shown in left and middle plots of Figure~\ref{fig:4-color-expls}. Most methods attributed incorrect importance to the {\it red} concept when it is missing (left extreme of left plot), and similarly for the {\it green} concept (left extreme of middle plot). 
Such explanations would have led the user to believe that {\it green} is more important than {\it red}, or {\it red} is more important than {\it green}, depending on the probe-dataset used (as shown in the right most plot). In practice, mis-communicated concept importance could be disastrous, if such importances are used to inform decision making and/or discovery. 
Because \ours{} also informs the user of uncertainty in the estimated importance, we see that the difference in importance scores between the two colors at either extremes is not statistically significant as shown in the rightmost plot. 

\begin{figure}[htb]
\centering
\vspace{-4mm}
  \includegraphics[width=0.5\linewidth]{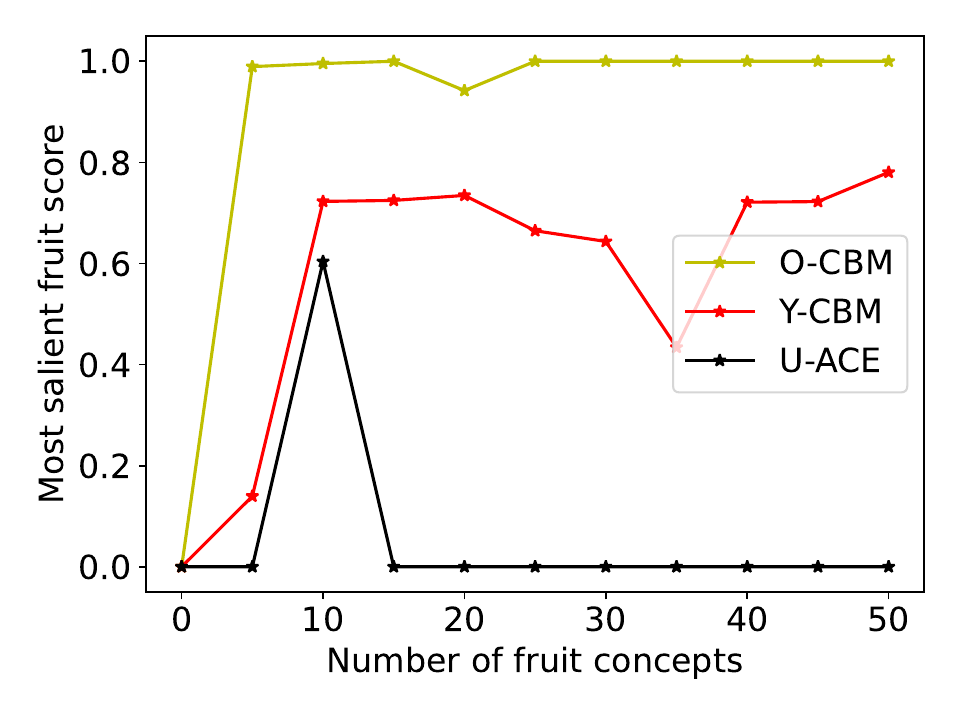}
  \caption{\ours{} is reliable even with overly complete concept set.}
  \label{fig:4-color-motive}
\end{figure}
{\bf Over-complete concept set}. 
We now evaluate the quality of explanations when the concept set is misspecified. We induce mispecification by making the concept set over-complete by gradually expanding it to include common fruit names 
(Appendix~\ref{sec:appendix:misc} contains the full list), which are clearly irrelevant to the task. We obtain the explanations using an in-distribution probe-dataset that contains all colors in equal proportion. Figure~\ref{fig:4-color-motive} shows the score of most salient fruit concept with increasing number of fruit (nuisance) concepts on X-axis. We observe that \ours{} is far more robust to the presence of nuisance concepts. Robustness to irrelevant concepts is important because it allows the user to begin with a superfluous set of concepts and find their relevance to \mex{}, instead of requiring users to guess relevant concepts, which is ironically the very purpose of using concept explanations.

{\bf Under-complete concept set}. 
We now generate concept explanations with concepts set to \{{\it ``red or blue'', ``blue or red'', ``green or blue'', ``blue or green''}\}. The concept {\it ``red or blue''} is expected to be active for both {\it red} or {\it blue} colors, similarly for {\it ``blue or red''} concept. 
Since all the concepts contain a color from each label, i.e. are active for both the labels, none of them must be useful for prediction. 
Yet, the importance scores  estimated by Y-CBM and O-CBM shown in the Figure~\ref{fig:4-color-motive2} table attribute significant importance. \ours{} avoids this problem as explained in Section~\ref{sec:method:analysis} and attributes almost zero importance. 

\begin{table}[htb]
\centering
  \begin{tabular}{c|r|r|r}
    Concept & Y-CBM & O-CBM & \ours{} \\\hline
    red or blue & -75.4 & -1.8 & 0.1 \\
    blue or red & 21.9 & -1.9 & 0 \\
    green or blue & -1.4 & 1.6 & 0 \\
    blue or green & -23.1 & 1.6 & 0 \\
  \end{tabular}
 \caption{When the concept set is under-complete and contains only nuisance concepts, their estimated importance score must be 0.}
  \label{fig:4-color-motive2}
\end{table}

\paragraph{List of fruit concepts from Section~\ref{sec:simstudy}.}
\begin{verbatim}
   apple, apricot, avocado, banana, blackberry, blueberry, cantaloupe,
   cherry, coconut, cranberry, cucumber, currant, date, dragonfruit,
   durian, elderberry, fig, grape, grapefruit, guava, honeydew, kiwi,
   lemon, lime, loquat, lychee, mandarin orange, mango, melon, nectarine,
   orange, papaya, passion fruit, peach, pear, persimmon, pineapple, plum,
   pomegranate, pomelo, prune, quince, raspberry, rhubarb, star fruit,
   strawberry, tangerine, tomato, watermelon 
\end{verbatim}

\paragraph{List of animal concepts from Section~\ref{sec:expt:stl}.}
\begin{verbatim}
    lion, tiger, giraffe, zebra, monkey, bear, wolf, fox, dog, cat, 
    horse, cow, pig, sheep, goat, deer, rabbit, raccoon, squirrel, mouse,
    rat, snake, crocodile, alligator, turtle, tortoise, lizard, 
    chameleon, iguana, komodo dragon, frog, toad, turtle, tortoise, 
    leopard, cheetah, jaguar, hyena, wildebeest, gnu, bison, antelope, 
    gazelle, gemsbok, oryx, warthog, hippopotamus, rhinoceros, elephant 
    seal, polar bear, penguin, flamingo, ostrich, emu, cassowary, kiwi, 
    koala, wombat, platypus, echidna, elephant
\end{verbatim}

\paragraph{Concepts used for {\it car} and {\it plane} from Section~\ref{sec:expt:stl}}
\begin{verbatim}
    car: headlights, taillights, turn signals, windshield, windshield vipers, 
         bumpers, wheels
    plane: wings, landing gear, sky
\end{verbatim}
\paragraph{Scene labels considered in Section~\ref{sec:expt:real}.}
\begin{verbatim}
    /a/arena/hockey, /a/auto_showroom, /b/bedroom, /c/conference_room, /c/corn_field
    /h/hardware_store, /l/legislative_chamber, /t/tree_farm, /c/coast, 
    /p/parking_lot, /p/pasture, /p/patio, /f/farm, /p/playground, /f/field/wild
    /p/playroom, /f/forest_path, /g/garage/indoor
    /g/garage/outdoor, /r/runway, /h/harbor, /h/highway
    /b/beach, /h/home_office, /h/home_theater, /s/slum, 
    /b/berth, /s/stable, /b/boat_deck, /b/bow_window/indoor, 
    /s/street, /s/subway_station/platform, /b/bus_station/indoor, /t/television_room,
    /k/kennel/outdoor, /c/campsite, /l/lawn, /t/tundra, /l/living_room, 
    /l/loading_dock, /m/marsh, /w/waiting_room, /c/computer_room, 
    /w/watering_hole, /y/yard, /n/nursery, /o/office, /d/dining_room, /d/dorm_room,
    /d/driveway
\end{verbatim}

\subsection{Further details for Imagenet experiment of Section~\ref{expt:simagenet}}
\label{appendix:simagenet}
Table~\ref{tab:simagenet:spurious} contains the list of labels we considered for evaluation on the Imagenet dataset of Section~\ref{expt:simagenet}.

Table~\ref{tab:simagenet_examples:more} extends with more results the Table~\ref{tab:simagenet_example} of the main content. 
\begin{table}[htb]
    \centering
    \begin{tabular}{c|l|l}
    \hline
    Index & Name & Known spurious features\\\hline
    1 & dogsled & \texttt{snow, dog, tree, trees, husky}\\
2 & howler monkey & \texttt{trunk, green, branches, branch, vegetation}\\
3 & seat belt, seatbelt & \texttt{passenger, window, sunglasses, van}\\
4 & ski & \texttt{tree, trees, snow, mountain, person, sunglasses}\\
5 & volleyball & \texttt{sand, players, player, setter, scoreboard, net}\\
6 & boathouse & \texttt{lake, water, dock, boat, shore}\\
7 & bee & \texttt{flower, daisy, petals}\\
8 & plate & \texttt{food, table, dining}\\
9 & barracouta, snoek & \texttt{person, face, hands, hand, sunglasses, cap}\\
10 & llama & \texttt{hay, grass, green, greens}\\
11 & rhinoceros beetle & \texttt{hand, head, palm, person, fingers}\\
12 & dowitcher & \texttt{water, reflection, lake, shoal, sandbar}\\
13 & white wolf, Arctic wolf& \texttt{net, fence, fencing}\\
14 & dragonfly & \texttt{blurry, green, flower, plant}\\
15 & gorilla, Gorilla gorilla & \texttt{green, tree, grass, trunk}\\
16 & shovel & \texttt{snow}\\
17 & doormat, welcome mat & \texttt{door}\\
18 & ruddy turnstone, Arenaria interpres & \texttt{mud, shore, sand, land, seashore, beach}\\
19 & albatross, mollymawk & \texttt{water, sea, ocean}\\
20 & sax, saxophone & \texttt{player, players, playing}\\
21 & balance beam, beam & \texttt{player, person, sport, arms, legs}\\
22 & bathing cap, swimming cap & \texttt{face, chest, person, swimmer, diver, gymnast}\\
23 & puck, hockey puck & \texttt{player, bat, ice, arena}\\
24 & dining table, board & \texttt{chairs, chair, corner}\\
25 & rugby ball & \texttt{ground, green, players, player}\\
26 & dock, dockage, docking facility & \texttt{boat, boats, ship, yacht, water, sea, lake}\\
27 & padlock & \texttt{chain, chains, door}\\
28 & potter's wheel & \texttt{hands, hand, person, face, head}\\
29 & ping-pong ball & \texttt{player, human, hands, arms, arm}\\
30 & paddle, boat paddle & \texttt{human, arm, arms, body, lifevest, water}\\
31 & unicycle, monocycle & \texttt{road, body, human, arms, face, arm}\\
32 & ice lolly, lolly, lollipop, popsicle & \texttt{mouth, eyes, face, human, head, hand, lip, lips}\\
33 & beaver & \texttt{water, lake, waterbody}\\
34 & mountain tent & \texttt{mountain, lake, water, hill, snow}\\
35 & indri, indris, Indri indri, Indri brevicaudatus & \texttt{sky, tree, leaf, leaves, trunk, vegetation, green}\\
36 & seashore, coast, seacoast, sea-coast & \texttt{seawater, ocean, water, sea}\\
37 & sunglass & \texttt{cheeks, face, head, person, nose}\\
38 & bulbul & \texttt{branch, tree, leaves, sky, leaf}\\
39 & alp & \texttt{sky, clouds, blue}\\
40 & Arabian camel & \texttt{desert, sand, hot, water, ground}\\\hline
    \end{tabular}
    \caption{Classes selected for evaluation in Section~\ref{expt:simagenet} and known spurious features in the last column. The known spurious features are obtained from \href{https://salient-imagenet.cs.umd.edu/}{Salient-Imagenet}}
    \label{tab:simagenet:spurious}
\end{table}

\begin{table}[htb]
    \centering
    \begin{tabular}{l|c|l}
        volley ball & Y-CBM & \texttt{setter, ball, \color{red}{airship}}\\
         &O-CBM & \texttt{setter, ball, \color{red}{clay}}\\
         &\ours{} & \texttt{sports, setter, ball}\\\hline
        Gorilla & Y-CBM & \texttt{animals, siamang, \color{red}{mower}} \\
         & O-CBM & \texttt{chimp, black, \color{red}{mastiff}} \\
         & \ours{} & \texttt{siamang, gibbon, macaque} \\\hline
        Rugby & Y-CBM & \texttt{rugby, knee, \color{red}{missile}} \\
        \hspace{3mm}ball & O-CBM & \texttt{rugby, soccer, \color{red}{bulldog}} \\
         & \ours{} & \texttt{rugby, sports, player} \\\hline
        hockey & Y-CBM & \texttt{hockey, player, \color{red}{nail}} \\
        \hspace{3mm}puck & O-CBM & \texttt{hockey, \color{red}{shaft, heater}} \\
         & \ours{} & \texttt{hockey, sports, \color{red}{gandola}} \\\hline
        swimming & Y-CBM & \texttt{hat, \color{red}{regions, bandaid}} \\
        \hspace{3mm}cap & O-CBM & \texttt{swimming, \color{red}{bald, meets}} \\
         & \ours{} & \texttt{swimming, sports, snorkel} \\\hline
    \end{tabular}
    \caption{More results extending the Table~\ref{tab:simagenet_example}}
    \label{tab:simagenet_examples:more}
\end{table}

\subsection{Addition results for Section~\ref{sec:expt:real}}
We report also the tau~\citep{wiki:Kendall_tau_distance} distance from concept explanations computed by \simple{} as a measure of explanation quality in Table~\ref{tab:expte}. Kendall Tau is a standard measure for measuring distance between two ranked lists. It does so my computing number of pairs with reversed order between any two lists. Since \simple{} can only estimate the importance of concepts that are correctly annotated in the dataset, we restrict the comparison to only over concepts that are attributed non-zero importance by \simple{}. 

\begin{table}[htb]
    \centering
    \begin{tabular}{c|r|r|r|r}
    Dataset$\downarrow$ & TCAV & O-CBM & Y-CBM & \ours{}\\\hline
       ADE20K  & 0.36  & 0.48  & 0.48 & {\bf 0.34} \\
       PASCAL  & 0.46  & 0.52  & 0.52 & {\bf 0.32}
    \end{tabular}
    \caption{{\it Quality of explanation comparison.} Kendall Tau Distance between concept importance rankings computed using different explanation methods shown in the first row with ground-truth. The ranking distance is averaged over twenty labels. \ours{} is better than both Y-CBM and O-CBM as well as TCAV despite not having access to ground-truth concept annotations.}
    \label{tab:expte}
\end{table}


\section{Ablation Study}
\label{sec:appendix:ablation}
\subsection{Uncertainty of Concept Activations}
\label{sec:appendix:uncert}
As explained in Section~\ref{sec:method:noise}, we estimate the uncertainty on concept activations using a measure on predictability of the concept as shown in Proposition~\ref{prop:1}. In this section we evaluate the quality of estimated uncertainty and compare with other (simpler) variants of uncertainty estimation. More crisply, we ask the following question.\newline 
{\it Why not estimate uncertainty using any other uncertainty quantification method?}

\paragraph{We will now introduce two new variants of estimating uncertainty.}
Alongside the measure of uncertainty of \ours{} discussed in Section~\ref{sec:method:noise}, which we denote with $\epsilon$, we will present two alternate ways of measuring uncertainty below. \newline
{\bf MC Sampling.} We may simply repeat the estimation procedure several times (denote by S) with different seed and data split to sample multiple concept activation vectors: $\{a_k^{(1)}, a_k^{(2)}, \dots , a_k^{(S)}\}\quad k\in [1, K]$. 
We empirically estimate per-concept uncertainty by averaging over examples: $\epsilon^{MC}=\mathbb{E}_{\vx\in \mathcal{D}}[V([\vx^T a_k^{(1)}, \vx^T a_k^{(2)}, \dots, \vx^T a_k^{(S)}])]$. Where $V(\bullet)$ is the sample variance: $V(b_1, b_2, \dots, b_S)=\frac{\sum_s (b_s-\frac{\sum_s b_s}{S})^2}{S-1}$. We simply repeated the estimation procedure of \citet{Oikarinen2023} that is summarized in Section~\ref{sec:method:noise} multiple times with different seed and data split to sample different activation vectors. 

{\bf Distribution Fit}. Inspired by ProbCBM proposed in \citet{KimJung23}, we estimate uncertainty from the data as a learnable parameter through distribution fitting. We assume a normal distrbution of the noise and model the standard deviation as a linear projection of the feature vector. The model is summarized below. 
\begin{align*}
    g(\vx)^Tg_{text}(T_k) \sim \mathcal{N}(\mu_k(\vx), \sigma_k^2(\vx))\\
    \mu_k(\vx) = \vec{p_k}^Tf(\vx), \sigma_k(\vx) = \vec{q_k}^Tf(\vx)
\end{align*}
We obtain the observed score of a concept given an example: $\vx$ using the multi-modal model: $g, g_{text}$ and the text description of the $k^{th}$ concept $T_k$. The concept score is modeled to be distributed by a normal distribution whose mean and standard deviation are linear functions of the feature representation of the \mex{}: $f^{[-1]}(\vx)$. We optimize the value for $[\vec{p_1}, \vec{p_2},\dots, \vec{p_K}], [\vec{q_1}, \vec{q_2},\dots,\vec{q_K}]$ through gradient descent on the objective $\mathcal{\beta} = MLL(\mathcal{D}, g) + \beta \times\mathbb{E}_\mathcal{D}\mathbb{E}_k[KL(\mathcal{N}(0, I)\| \mathcal{N}(\mu_k(\vx), \sigma_k(\vx)))]$ very similar to the proposal of \citet{KimJung23}. We picked the best value of $\beta$ and obtained $\epsilon^{DF}$ by averaging over all the examples: $\epsilon^{DF}=\mathbb{E}_{\mathcal{D}}[\sigma_k(\vx)]$.

\paragraph{Evaluation of different uncertainties.}
We conduct our study using the ResNet-18 model pretrained on Places365 and Pascal dataset that were discussed in Section~\ref{sec:expt:real}. We use human-provided concept annotations to train per-concept linear classifier on the representation layer. We retained only 215 concepts of the total 720 concepts that have at least two positive examples in the dataset. We then evaluated the per-concept linear classifier on a held-out test set to obtain macro-averaged accuracy. The concepts with poor accuracy are the ones that cannot be classified linearly using the representations. Therefore the error rate per concept is the ground-truth for uncertainty that we wish to quantify. 

We may now evaluate the goodness of uncertainty: $\epsilon$ of \ours{} by comparing it with ground-truth (error-rate); observe they are both K-dimensional vectors. We report two measures of similarity in Table~\ref{tab:uncert}: (1) Cosine-Similarity (Cos-Sim) between $\epsilon$ and error-rate, (2) Jaccard Similarity (JS) (\url{https://en.wikipedia.org/wiki/Jaccard_index}) between top-k least uncertain concepts identified using error-rate and $\epsilon$.  
For any two vectors $u, v$, and their top-k sets: $S_1(u), S_2(v)$, the Cos-Sim and JS are evaluated as follows. 
\begin{align*}
    &\text{Cos-Sim}(u, v) = \frac{u^Tv}{\|u\|\|v\|}\\
    &JS(S_1(u), S_2(v)) = \frac{|S_1(u)\cap S_2(v)|}{|S_1(u)\cup S_2(v)|}
\end{align*}

{\bf Evaluation of epistemic uncertainty.} We compared the estimate of uncertainty obtained through MC sampling ($\epsilon^{MC}$ with hundred samples) and Distribution Fitting ($\epsilon^{DF}$) with $\epsilon$ of \ours{} in Table~\ref{tab:uncert}. We observe that uncertainty obtained using distributional fitting is decent without incurring huge computational cost, however \ours{} produced the highest quality uncertainty at the same or slightly lower computational cost of distributional fitting.  

\begin{table}[htb]
    \centering
    \begin{tabular}{|c|r|r|r|r|}
        \hline
         Method & Cos-Sim & Top-10 & Top-40 & Top-80 \\\hline
         MC Sampling & -0.13 & 0 & 0.08 & 0.21\\
         Distribution Fit & 0.06 & 0.11 & 0.19 & 0.31 \\
         U-ACE & {\bf 0.36} & 0.11 & {\bf 0.29} & {\bf 0.36}\\\hline
    \end{tabular}
    \caption{Evaluation of uncertainties estimated using U-ACE, MC sampling and Distribution Fit (see text for their description). Cos-Sim is the cosine-similarity with ground-truth value of uncertainty. The next three columns show Jaccard similarity between the top-k concepts ranked by ground-truth uncertainty and each of the three methods. Higher the better for all the values. }
    \label{tab:uncert}
\end{table}

{\bf Evaluation of uncertainty due to ambiguity.} The results so far have confirmed the merits of \ours{} over the other two in modelling the uncertainty due to lack of information. In Figures~\ref{fig:amb:1},\ref{fig:amb:2},\ref{fig:amb:3},\ref{fig:amb:4}, we present anecdotal evidence that \ours{} is very effective at modelling uncertainty due to ambiguity. In each figure, we compare most (first two columns) and least (last two columns) uncertain images identified by Distribution Fit (in the first row) and \ours{} in the second row. 

\begin{figure}[htb]
    \centering
    \includegraphics[width=0.95\textwidth]{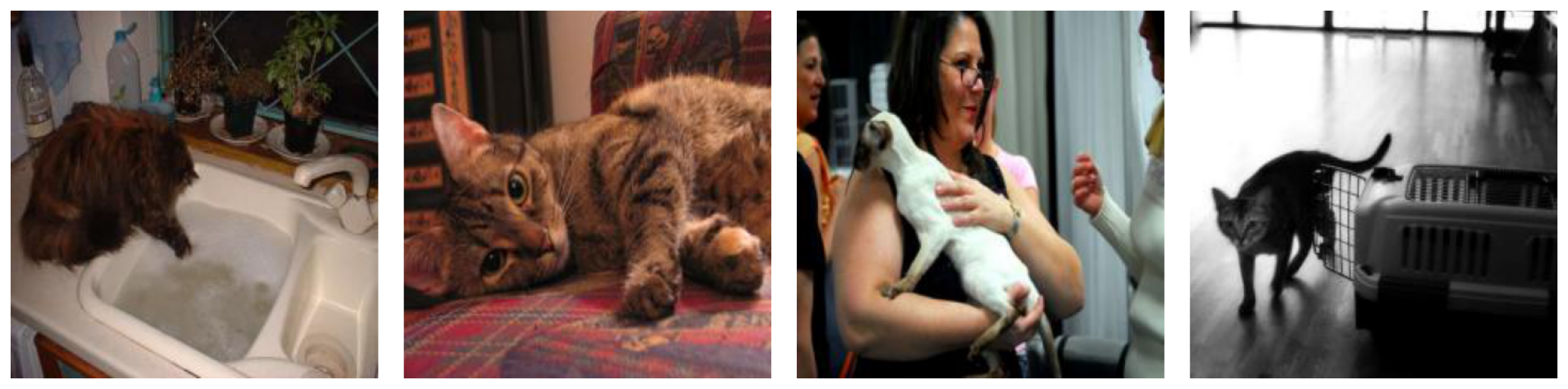}\\
    \includegraphics[width=0.95\textwidth]{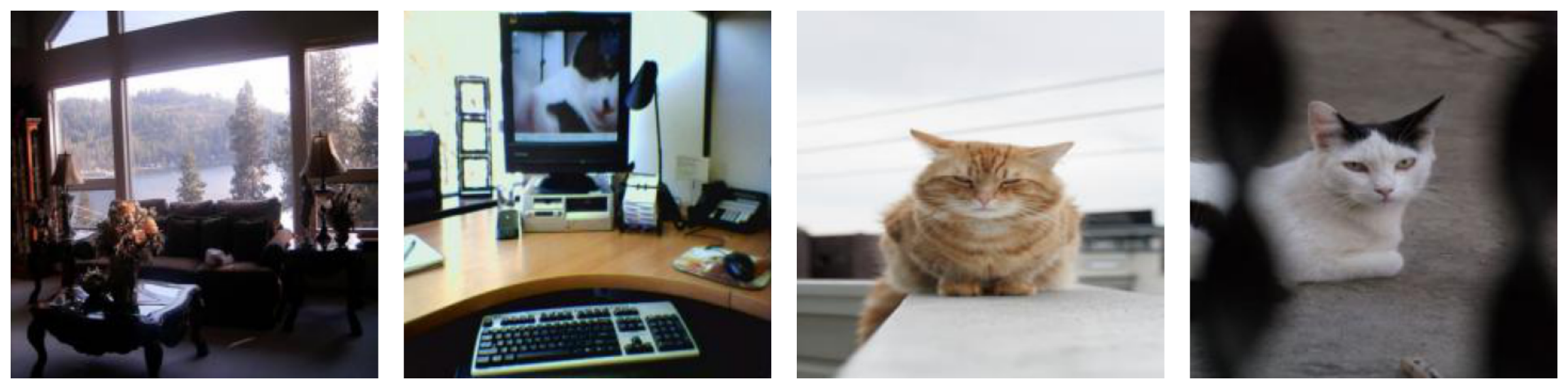}
    \caption{Comparison of ambiguity ranking for {\bf Cat} with Distr. Fit in the top row and \ours{} in the bottom row. Most uncertainty (due to ambiguity) on the left to least uncertainty on the right.}
    \label{fig:amb:1}
\end{figure}

\begin{figure}[htb]
    \centering
    \includegraphics[width=0.95\textwidth]{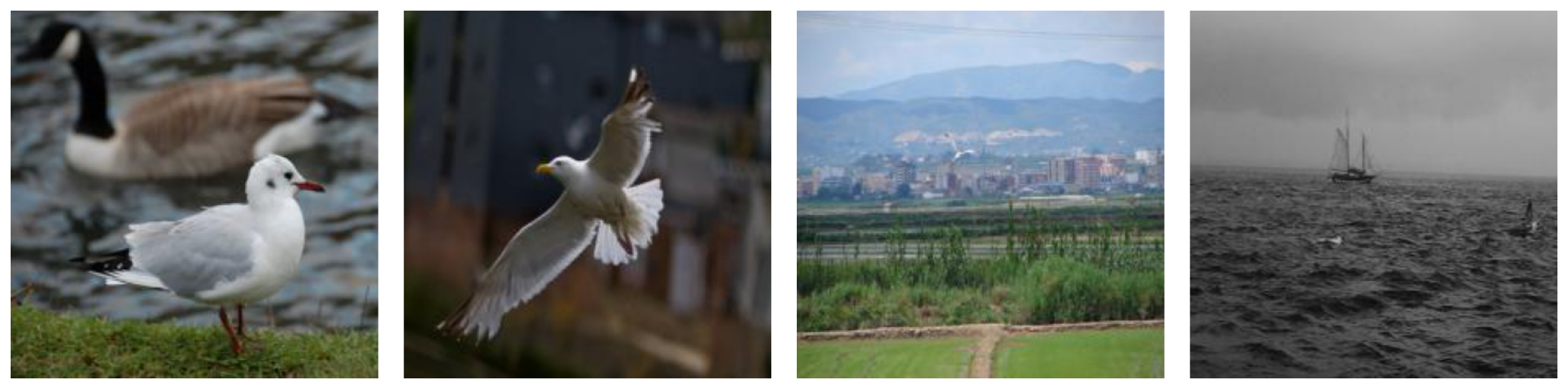}\\
    \includegraphics[width=0.95\textwidth]{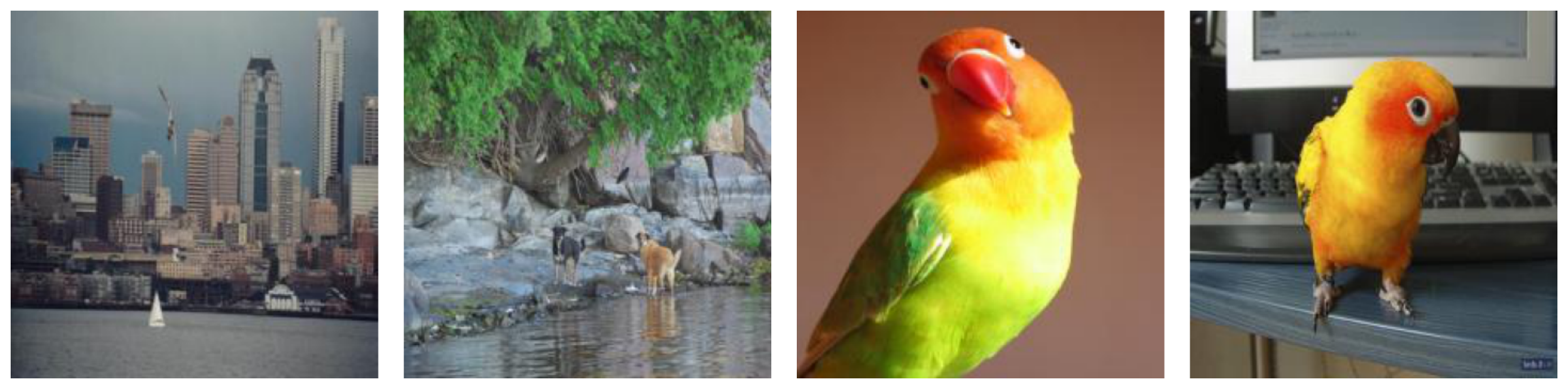}
    \caption{Comparison of ambiguity ranking for {\bf Bird} with Distr. Fit in the top row and \ours{} in the bottom row. Most uncertainty (due to ambiguity) on the left to least uncertainty on the right.}
    \label{fig:amb:2}
\end{figure}

\begin{figure}[htb]
    \centering
    \includegraphics[width=0.95\textwidth]{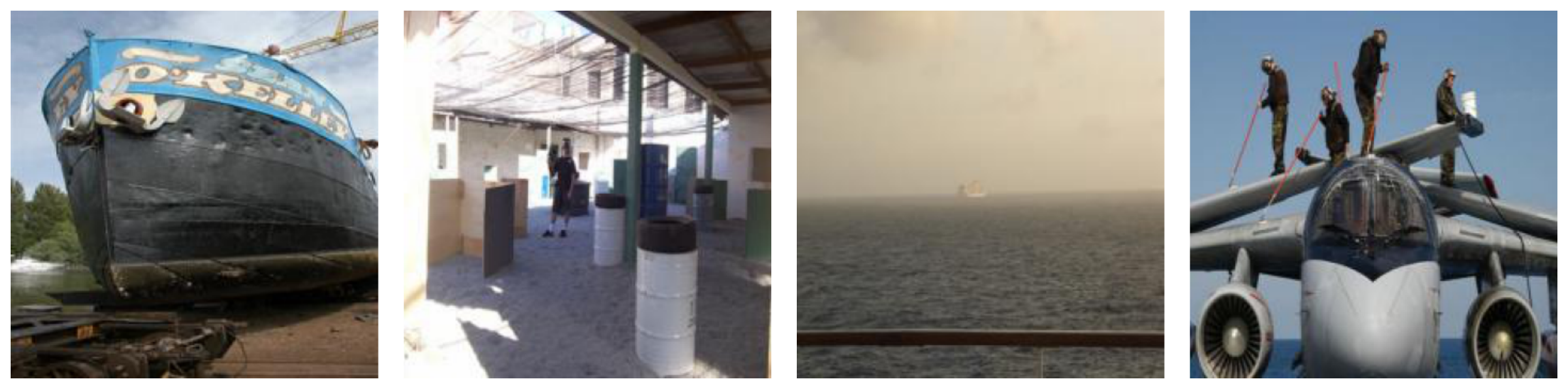}\\
    \includegraphics[width=0.95\textwidth]{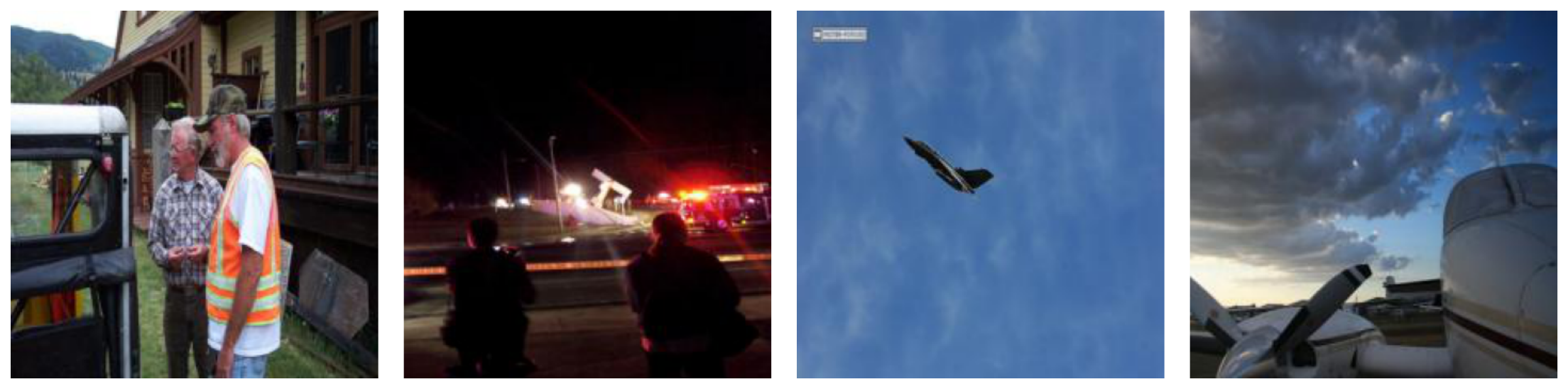}
    \caption{Comparison of ambiguity ranking for {\bf Sky} with Distr. Fit in the top row and \ours{} in the bottom row. Most uncertainty (due to ambiguity) on the left to least uncertainty on the right.}
    \label{fig:amb:3}
\end{figure}

\begin{figure}[htb]
    \centering
    \includegraphics[width=0.95\textwidth]{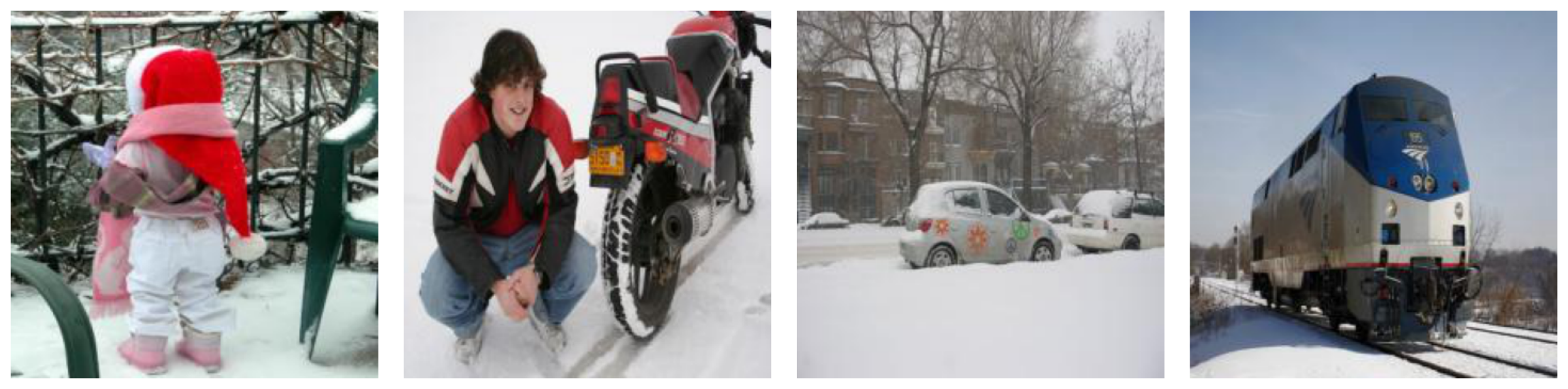}\\
    \includegraphics[width=0.95\textwidth]{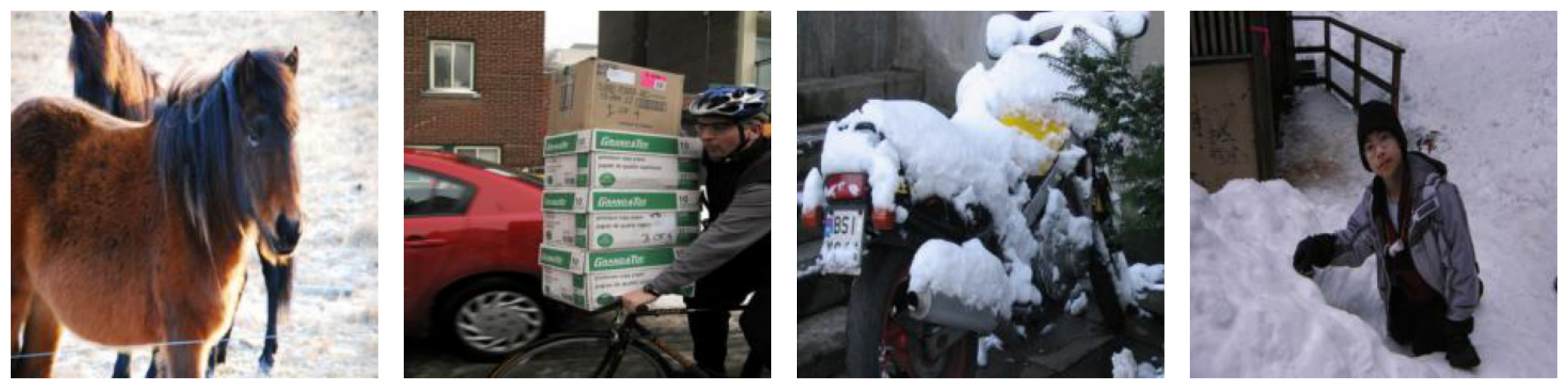}
    \caption{Comparison of ambiguity ranking for {\bf Snow} with Distr. Fit in the top row and \ours{} in the bottom row. Most uncertainty (due to ambiguity) on the left to least uncertainty on the right.}
    \label{fig:amb:4}
\end{figure}

\subsection{Bayesian Estimation and Significance of Prior}
The focus of this section is to motivate the uncertainty-aware prior used by \ours{}. More crisply, the subject of this section is to answer the following question. \newline
{\it What is the role of prior in~\ours{}, and what happens without it?}

We replicate the study on Broden dataset of Section~\ref{sec:expt:real} of Table~\ref{tab:real:expls} with two new baselines. We replace the linear model estimation of \ours{} described in Section~\ref{sec:method} with an out-of-the-box Bayesian Regression estimator available from \texttt{sklearn}\footnote{\url{https://scikit-learn.org/stable/modules/generated/sklearn.linear_model.BayesianRidge.html}}, which we refer as Bayes Regr. Effectively, Bayes Regr. is different from \ours{} only in the prior. We also compare with the estimation of fitting using Bayes Regr. but when the input is perturbed with the noise estimated by \ours{}. We refer to this baseline as Bayes Regr. with MC. 

Table~\ref{tab:prior} contrasts the two methods that differ majorly only on the choice of prior with \ours{}. We observe a drastic reduction in the quality of explanations by dropping the prior.   

{\it Can we trivially fix TCAV with a simple uncertainty estimate?}\newline
TCAV is already equipped with a simple uncertainty measurement to distinguish a truly important concept from a random concept. TCAV computes $\vec{m(x)}$ and $\vec{s(x)}$ of concept activations by simply training multiple concept activation vectors. Yet, TCAV estimated explanations are noisy as seen in Table 1 and in the top-10 salient concepts shown below. The poor quality of TCAV explanations despite employing uncertainty (although in a limited capacity) is likely because simple measurement of uncertainty through MC sampling is not the best method for estimating uncertainty as shown in Table~\ref{tab:uncert}.

\begin{table}[htb]
    \centering
    \begin{tabular}{|c|c|c|c|}
        \hline
         Dataset & Bayes Regr. &  Bayes Regr. with MC & \ours{}  \\\hline
         ADE20K & 0.39 & 0.43 & {\bf 0.09} \\ 
         Pascal & 0.40 & 0.45 & {\bf 0.11} \\\hline
    \end{tabular}
    \caption{Significance of prior: quality of explanations severely degrades without the uncertainty-aware prior.}
    \label{tab:prior}
\end{table}

The tables below give a detailed view of the top-10 salient concepts identified using ADE20K for the ResNet-18 scene classification model. The problematic or outlandish concepts are marked in red. We observe that although Bayesian Regr. and Y-CBM are practically the same, the choice of the estimator and sparsity seems to have helped Y-CBM produce (seemingly) higher quality explanations. 

Label: {\bf Tree Farm} (ADE20K)

\begin{tabular}{c|c}
    TCAV & \makecell{\texttt{{\color{red} palm}, horse, {\color{red} pane of glass}, {\color{red} helicopter},}\\ \texttt{rubbish, {\color{red} cap, boat, organ}, tent, footbridge}}\\\hline
    \makecell{Bayes Regr.\\with MC sampling} & \makecell{\texttt{{\color{red} net, merchandise, labyrinth, black, big top, ottoman, chest},}\\ \texttt{pigeonhole, tree, sky}}\\\hline
    Bayes Regr. & \makecell{\texttt{oar, forest, pigeonhole, {\color{red} merchandise}, sand trap,}\\ \texttt{net, {\color{red} wallpaper, tray, calendar}, tree}}\\\hline
    O-CBM & \makecell{\texttt{forest, pot, pottedplant, hedge, trestle, {\color{red} sweater},}\\ \texttt{bush, leaf, foliage, {\color{red} coat}}}\\\hline
    Y-CBM & \makecell{\texttt{field, forest, foliage, {\color{red} elevator}, gravestone,}\\\texttt{hedge, bush, vineyard, covered bridge, {\color{red} baptismal font}}}\\\hline
    \ours{}& \makecell{\texttt{foliage, forest, grass, field, hedge, covered bridge,}\\\texttt{tree, leaves, bush, gravestone}}
\end{tabular}

Label: {\bf Coast} (ADE20K)

\begin{tabular}{c|r}
     TCAV & \makecell{\texttt{{\color{red} shutter, manhole, baby buggy}, umbrella, sand, boat}, \\\texttt{arch, minibike, rubbish, {\color{red} column}}} \\\hline
    \makecell{Bayes Regr.\\with MC sampling} & \makecell{\texttt{{\color{red} wineglass, guitar, headlight, chest, jersey, roundabout,}}\\\texttt{\color{red} witness stand, magazine, folding door, shaft}}\\\hline
    Bayes Regr. & \makecell{\texttt{lake, {\color{red} headlight}, island, {\color{red} hen},}\\ \texttt{dog, {\color{red} chest, jersey, mosque, shaft, windshield}}}\\\hline
    O-CBM & \makecell{\texttt{sea, island, lighthouse, cliff, wave, shore,}\\ \texttt{rock, sand, {\color{red} pitted, crystalline}}}\\\hline
    Y-CBM & \makecell{\texttt{sea, sand, lake, island, {\color{red} runway}, cliff,}\\ \texttt{fog bank, clouds, {\color{red} towel rack}, pier}}\\\hline
    \ours{} & \makecell{\texttt{sea, lake, island, pier, cliff, lighthouse,}\\ \texttt{shore, fog bank, water, sand}}\\\hline
\end{tabular}

\subsection{Effect of Regularization Strength on Y-CBM and O-CBM}
We present the sensitivity analysis for the two strong baselines: Y-CBM and O-CBM in this section.

{\it How are the hyperparams tuned?}\newline
Hyperparameter tuning is tricky for concept explanations since they lack a ground-truth or validation set. The reported results for Y-CBM and O-CBM in the main paper used the default value of the regularization strength of the corresponding estimator, which is \texttt{Lasso}\footnote{\url{https://scikit-learn.org/stable/modules/generated/sklearn.linear_model.Lasso.html}} for Y-CBM and \texttt{LogisticRegression}\footnote{\url{https://scikit-learn.org/stable/modules/generated/sklearn.linear_model.LogisticRegression.html}} for O-CBM. Both the estimators are part of \texttt{sklearn}. We had to reduce the default regularization strength of Y-CBM to $\alpha=10^{-3}$ so that estimated weights are not all 0. The $\kappa$ of \ours{} is somewhat arbitrarily set to $0.02$ on Broden dataset for sparse explanation with non-zero weight for only 20-30\% of the concepts. 

{\it Can Y-CBM and O-CBM do much better if we tune the regularization strength?}\newline
We present the results of the two baselines for various values of the regularization strength in Table~\ref{tab:hparam}. The table shows quality of explanations in the first two rows for the same setup as Table~\ref{tab:real:expls} and also shows the measure of drift in explanations like in Table~\ref{tab:real:shift}. We tried C=\texttt{1e-2, 0.1, 1, 10} for O-CBM and $\alpha=$\texttt{1e-4, 1e-3, 1e-2, 1e-1} for Y-CBM. We dropped C=1e-2 and $\alpha=$\texttt{1e-2, 1e-1} from the table because then the exaplanations were overly sparsified to zero. 

We observe from the table that \ours{} still is the best method that that yields high-quality explanation while also being less sensitive to shift in the probe-dataset. 

\begin{table}[htb]
    \centering
    \begin{tabular}{|c|r|r|r|r|r|r|}
    \hline
    Dataset & \multicolumn{3}{|c|}{O-CBM} & \multicolumn{2}{|c|}{Y-CBM} & \ours{} \\\hline
    Regularization strength $\rightarrow$ & C=0.1 & C=1 & C=10 & $\alpha$=$10^{-4}$ & $\alpha$=$10^{-3}$ & $\kappa$=0.02 \\\hline
    ADE20K & 0.12 & 0.20 & 0.29 & 0.24 & 0.14 & {\bf 0.09}\\
    Pascal & 0.11 & 0.25 & 0.35 & 0.27 & 0.13 & {\bf 0.11}\\\hline
    ADE20K$\rightarrow$Pascal & 0.46 & 0.26 & 0.12 & 0.29 & 0.34 & {\bf 0.19}\\\hline
    \end{tabular}
    \caption{Results on Broden dataset with varying value of regularization strength for O-CBM and Y-CBM. ADE20K and Pascal rows compare the distance between the explanations computed from the ground-truth exactly like Table~\ref{tab:real:expls}. The last row compares how much the explanations drifted between the datasets exactly like in Table~\ref{tab:real:shift}. Lower is better everywhere. Observe that \ours{} has high explanation quality while also being relatively more robust to data shift.}
    \label{tab:hparam}
\end{table}

\section{Evaluation using CUB dataset}
\citet{WahCUB_200_2011} released a bird dataset called CUB with 11,788 images and 200 bird species. Moreover, each bird image is annotated with one of 312 binary attributes indicating the presence or absence of a bird feature. \citet{Koh2020} popularized an improved version of the dataset that retained only 112 clean attribute annotations. We evaluate using the cleaner dataset released by \citet{Koh2020} owing to their popularity in evaluating CBMs. We train a pretrained ResNet-18 model using the training split of CUB dataset. We then compute the explanation (i.e. saliency of concepts) using the test split. Similar to the evaluation of Section~\ref{sec:expt:real}, we quantify the quality of explanations using distance from explanations computed using true concept annotations when using \simple{}. 

\begin{table}
\centering
\begin{tabular}{|c|r|rrr|}
    & TCAV & O-CBM & Y-CBM & \ours{}\\ 
    \hline
    Top 3 & {\bf 0.38} & 0.43 & 0.46 & {\bf 0.43} \\
    Top 5 & {\bf 0.39} & 0.45 & 0.46 & {\bf 0.44} \\
    Top 10 & {\bf 0.37} & 0.43 & 0.45 & {\bf 0.42}\\
    Top 20 & {\bf 0.36} & 0.40 & 0.41 & {\bf 0.39} 
\end{tabular}
\caption{Distance of top-k salient concepts computed using \simple{} and different estimation methods shown in the first row (lower the better). TCAV does well overall and \ours{} performs the best among methods without access to concept annotations.}
\label{fig:cub:l1_dist}
\end{table}

\begin{table}
\centering
\begin{tabular}{|c|r|rrr|}
    & TCAV & O-CBM & Y-CBM & \ours{}\\    
    \hline
    Top 3 & {\bf 0.22} & 0.17 & 0.13 & {\bf 0.20} \\
    Top 5 & {\bf 0.5} & 0.33 & 0.28 & {\bf 0.45} \\
    Top 10 & {\bf 1.68} & 1.18 & 1.02 & {\bf 1.34} \\
    Top 20 & {\bf 5.16} & 4.185 & 3.935 & {\bf 4.51} 
\end{tabular}
\caption{Average overlap between top-k salient concepts computed using \simple{} and different estimation methods shown in the first row (higher the better). TCAV does well overall and \ours{} performs the best among methods without access to concept annotations.}
\end{table}

\end{document}